\documentclass{article}
\usepackage{matus}
\usepackage{bm}

\def\R{\mathbb{R}}
\def\N{\mathbb{N}}
\def\1{\mathds{1}}
\def\conv{\textup{conv}}

\def\cs{\bar{\partial}}
\def\csr{\cs_r}
\def\csp{\cs_\perp}
\def\nf{\nabla f}

\def\nL{\nabla\cL}
\def\na{\nabla\alpha}
\def\nra{\nabla_r\alpha}
\def\npa{\nabla_\perp\alpha}
\def\tW{\widetilde{W}}
\def\ta{\tilde{\alpha}}

\def\pc{k}
\def\setdef#1#2{\cbr{#1\,\middle|\,#2}}
\def\lexp{\ell_{\exp}}
\def\llog{\ell_{\log}}
\def\diag{\textup{diag}}
\def\eps{\epsilon}

\def\baru{\bar u}
\def\barq{\bar q}
\def\barp{\bar p}
\def\bars{\bar s}
\def\btheta{\bar \theta}
\def\T{{\scriptscriptstyle\mathsf{T}}}
\renewcommand{\S}{\mathbb{S}}
\def\oW{\overline{W}}

\def\cifar{\texttt{cifar}\xspace}

\title{Directional convergence and alignment in deep learning}

\iftrue
\author{Ziwei Ji\qquad{}Matus Telgarsky\\
\texttt{\{\href{mailto:ziweiji2@illinois.edu}{ziweiji2},\href{mailto:mjt@illinois.edu}{mjt}\}@illinois.edu} \\
University of Illinois, Urbana-Champaign}
\else
\author{Ziwei Ji\thanks{\tt{<\url{ziweiji2@illinois.edu}>} and \tt{<\url{mjt@illinois.edu}>};
University of Illinois, Urbana-Champaign.}
\and
Matus Telgarsky\footnotemark[1]}
\fi
\date{}

\begin{document}

\maketitle

\begin{abstract}
  In this paper, we show that although the minimizers of cross-entropy and
  related classification losses are off at infinity, network weights learned by
  gradient flow converge \emph{in direction}, with an immediate corollary that
  network predictions, training errors, and the margin distribution also
  converge.
  This proof holds for deep homogeneous networks
  ---
  a broad class of networks allowing for ReLU, max-pooling, linear, and
  convolutional layers
  ---
  and we additionally provide empirical support not just close to the theory
  (e.g., the AlexNet), but also on non-homogeneous networks (e.g., the
  DenseNet).
  If the network further has locally Lipschitz gradients, we show that these
  gradients also converge in direction, and asymptotically \emph{align} with the
  gradient flow path, with consequences on margin maximization, convergence of saliency maps,
  and a few other settings.
  Our analysis complements and is distinct from the well-known neural tangent
  and mean-field theories, and in particular makes no requirements on network
  width and initialization, instead merely requiring perfect classification
  accuracy.
  The proof proceeds by developing a theory of unbounded nonsmooth
  Kurdyka-{\L}ojasiewicz inequalities for functions definable in an o-minimal
  structure, and is also applicable outside deep learning.
\end{abstract}

\section{Introduction}\label{sec:intro}

Recent efforts to rigorously analyze the optimization of deep networks have
yielded many exciting developments, for instance the neural tangent
\citep{jacot_ntk,du_deep_opt,allen_deep_opt,zou_deep_opt}
and mean-field perspectives
\citep{montanari_mean_field,chizat_bach_meanfield}.
In these works, it is shown that small training or even testing error are
possible for wide networks.

The above theories, with finite width networks, usually require the weights to
stay close to initialization in certain norms.
By contrast, practitioners run their optimization methods as long as their
computational budget allows \citep{jascha_step_sizes}, and if the data can be
perfectly classified, the parameters are guaranteed to diverge in norm to
infinity \citep{kaifeng_jian_margin}.
This raises a worry that the prediction surface can continually change during
training;
indeed, even on simple data, as in \Cref{fig:contours}, the
prediction surface continues to change after perfect classification is achieved,
and even with large width is not close to the maximum margin predictor from the
neural tangent regime.
If the prediction surface never stops changing, then the generalization behavior,
adversarial stability, and other crucial properties of the predictor could also
be unstable.

In this paper, we resolve this worry by guaranteeing stable convergence behavior
of deep networks as training proceeds, despite this growth of weight vectors to
infinity.
Concretely:
\begin{enumerate}
  \item \textbf{Directional convergence:} the parameters converge \emph{in
  direction}, which suffices to guarantee convergence of many other relevant
  quantities, such as the \emph{prediction margins}.

  \item \textbf{Alignment:} when gradients exist, they converge in direction to
  the parameters, which implies various margin maximization results and saliency
  map convergence, to name a few.
\end{enumerate}

\subsection{First result: directional convergence}

\begin{figure}[t]
  \centering
  \begin{subfigure}[t]{0.31\textwidth}
    \centering
\includegraphics[width=\textwidth]{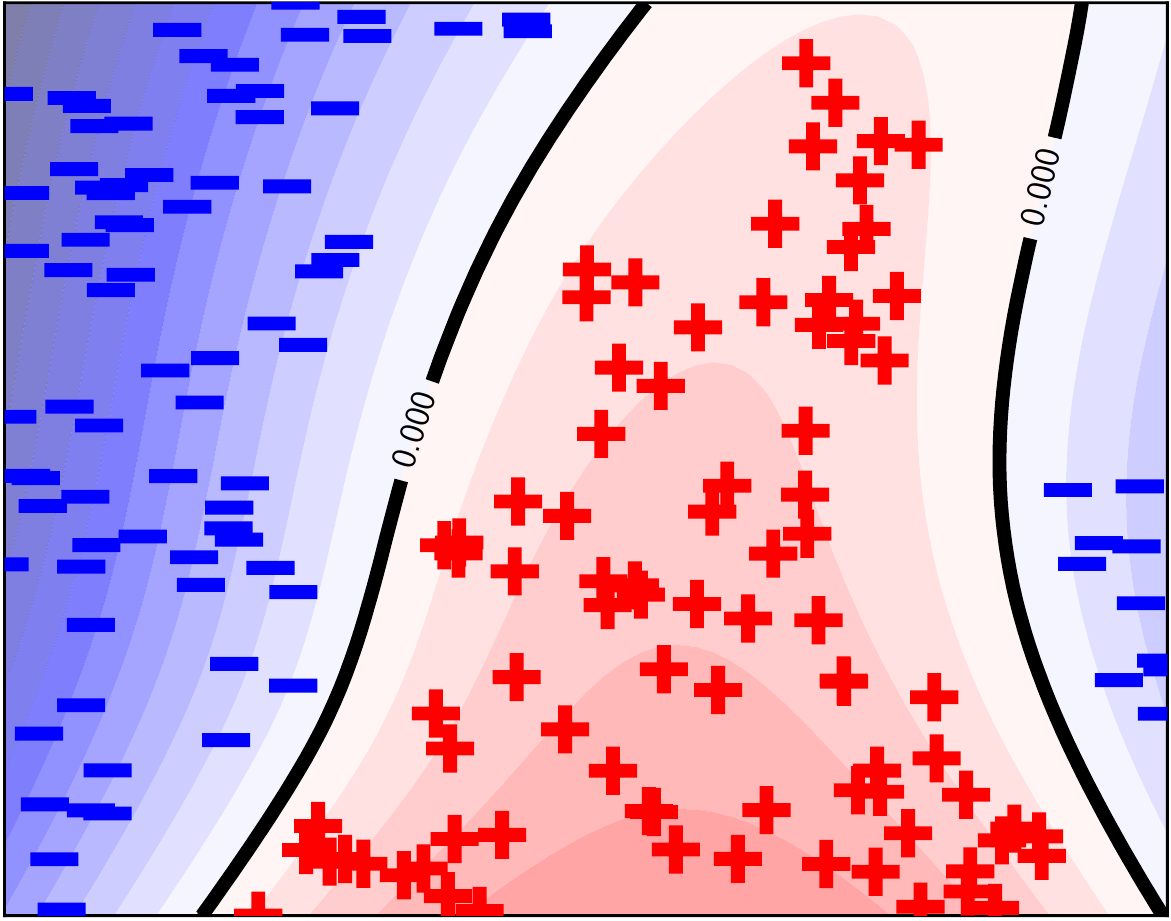}
    \caption{Shallow NTK max margin.}
    \label{fig:contours:ntk}
  \end{subfigure}\hfill
  \begin{subfigure}[t]{0.31\textwidth}
    \centering
\includegraphics[width=\textwidth]{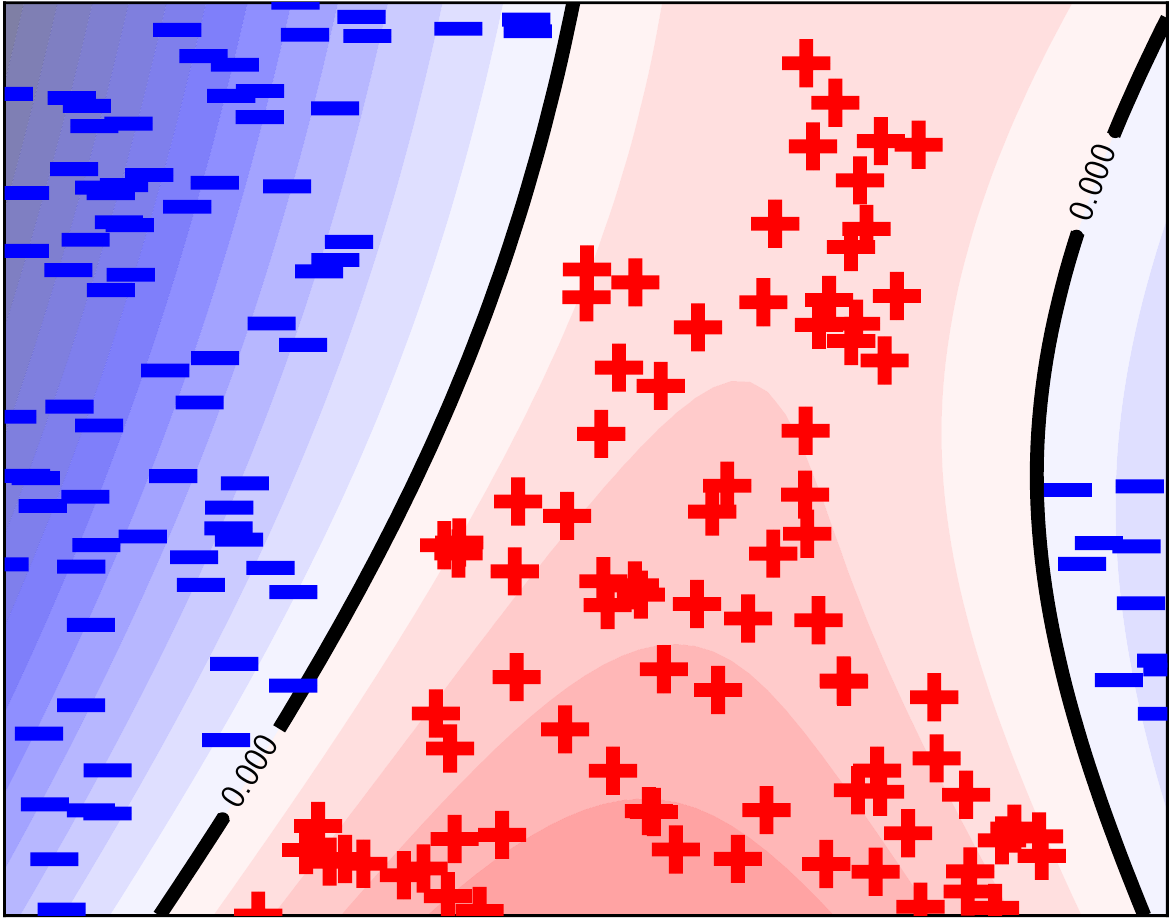}
\caption{Shallow net, early training.}
    \label{fig:contours:a}
  \end{subfigure}\hfill
\begin{subfigure}[t]{0.31\textwidth}
    \centering
\includegraphics[width=\textwidth]{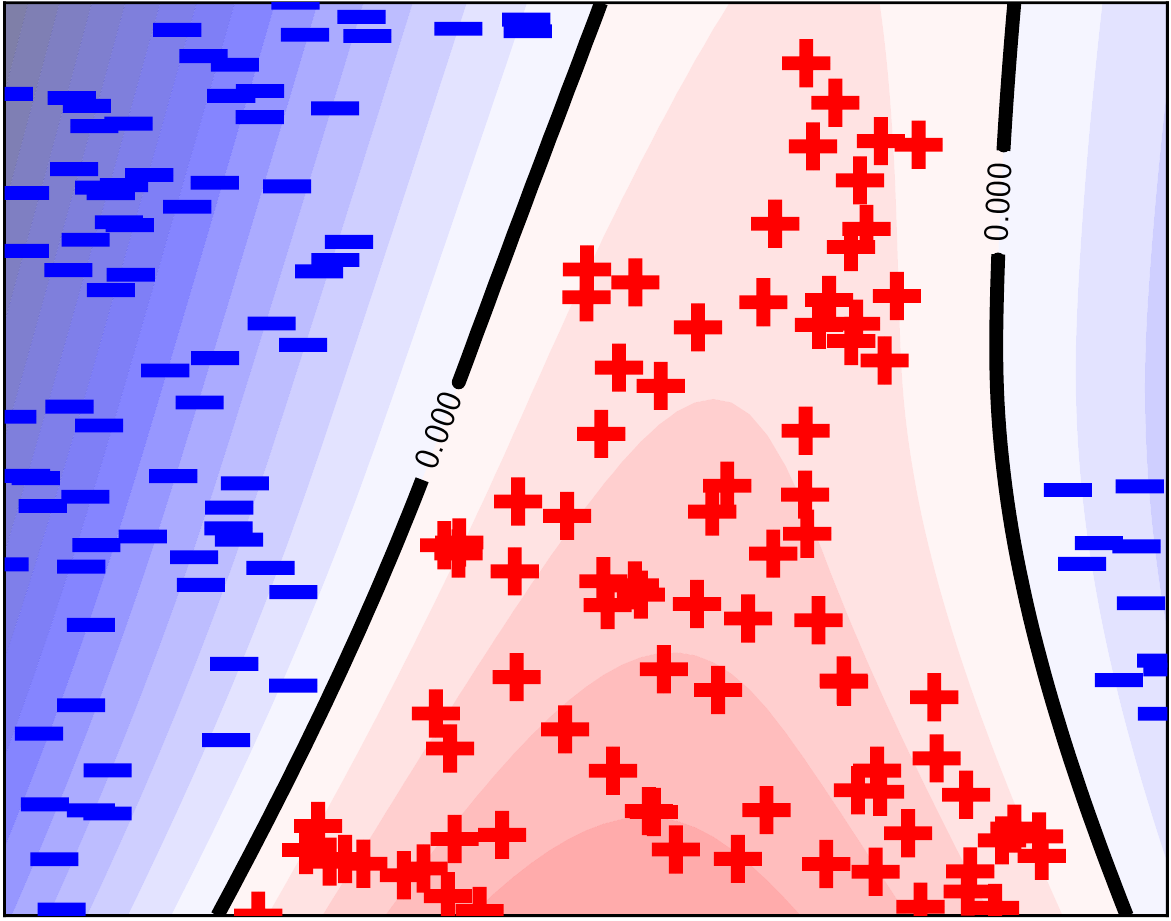}
\caption{Shallow net, late training.}
    \label{fig:contours:c}
  \end{subfigure}\caption{Prediction surface of a shallow network on simple synthetic data
    with blue negative examples (``{\color{blue}$\bm{-}$}'') and
    red positive examples (``{\color{red}$\bm{+}$}''), trained via gradient descent.
    \Cref{fig:contours:ntk} shows the prediction surface reached by freezing
    activations, which is also the prediction surface of the corresponding
    Neural Tangent Kernel (NTK) maximum margin predictor \citep{nati_logistic}.
    \Cref{fig:contours:a} shows the same network, but now without frozen activations,
    at the first moment with perfect classification.
    Training this network much longer converges to \Cref{fig:contours:c}.
  }
  \label{fig:contours}
\end{figure}

We show that the network parameters $W_t$ converge \emph{in direction}, meaning
the normalized iterates $\nicefrac {W_t}{\|W_t\|}$ converge.
Details are deferred to \Cref{sec:dir}, but here is a brief overview.

Our networks are \emph{$L$-positively homogeneous in the parameters}, meaning
scaling the parameters by $c>0$ scales the predictions by $c^L$, and
\emph{definable in some $o$-minimal structure},
a mild technical assumption which we will describe momentarily.
Our networks can be arbitrarily deep with many common types of layers (e.g.,
linear, convolution, ReLU, and max-pooling layers), but homogeneity rules out
some components such as skip connections and biases, which all satisfy
definability.

We consider binary classification with either the logistic loss
$\llog(z) := \ln(1+e^{-z})$ (binary cross-entropy) or the exponential loss
$\lexp(z):= e^{-z}$, and a standard gradient flow (infinitesimal gradient
descent) for non-differentiable non-convex functions via the Clarke
subdifferential.
We start from an initial risk smaller than $\nicefrac 1 n$, where $n$ denotes
the number of data samples; in this way, our analysis handles the late phase of
training, and can be applied after some other analysis guarantees risk
$\nicefrac 1 n$.

Under these conditions, we prove the following result, without any other
assumptions about the distribution of the parameters or the width of the network
(cf. \Cref{fact:dir}):
\begin{center}
    \em{The curve swept by $\nicefrac{W_t}{\|W_t\|}$ has finite length, and
    thus $\nicefrac{W_t}{\|W_t\|}$ converges.}
\end{center}

Our main corollary is that \emph{prediction margins} converge (cf.
\Cref{fact:margin_conv}), meaning convergence of the normalized per-example values
$\nicefrac {y_i \Phi(x_i;W_t)}{\|W_t\|^L}$, where $y_i$ is the label and
$\Phi(x_i;W_t)$ is the prediction on example $x_i$.
These quantities are central in the study of generalization of deep networks,
and their stability also implies stability of many other useful quantities
\citep{spec,margindist_predict,fantastic}.
As an illustration of directional convergence and margin convergence, we plot
the margin values for all examples in the standard \cifar data against training
iterations in \Cref{fig:margins}; these trajectories exhibit strong convergence
behavior, both within our theory
(a modified homogeneous AlexNet, as in \Cref{fig:margins:halexnet}),
and outside of it (DenseNet, as in \Cref{fig:margins:densenet}).

Directional convergence is often assumed throughout the
literature~\citep{GLSS18,chizat_bach_imp}, but has only been established for
linear predictors~\citep{nati_logistic}.
It is tricky to prove because it may still be false for highly smooth functions:
for instance, the homogeneous Mexican Hat function satisfies all our assumptions
\emph{except} definability, and can be adjusted to have arbitrary order of
continuous derivatives, but its gradient flow \emph{does not} converge in
direction, instead it spirals~\citep{kaifeng_jian_margin}.
To deal with similar pathologies in many branches of mathematics, the notion of
functions \emph{definable in some o-minimal structure} was developed: these are
rich classes of functions built up to limit oscillations and other bad
behavior.
Using techniques from this literature, we build general tools, in
particular unbounded nonsmooth Kurdyka-{\L}ojasiewicz inequalities, which
allows us to prove directional convergence, and may also be useful outside deep
learning.
More discussion on the o-minimal literature is given in \Cref{sec:rw}, technical
preliminaries are introduced in \Cref{sec:prelim}, and a proof overview is given
in \Cref{sec:dir}, with full details in the appendices.

\begin{figure}[t!]
  \centering
  \begin{subfigure}[t]{0.48\textwidth}
    \centering
    \includegraphics[width=\textwidth]{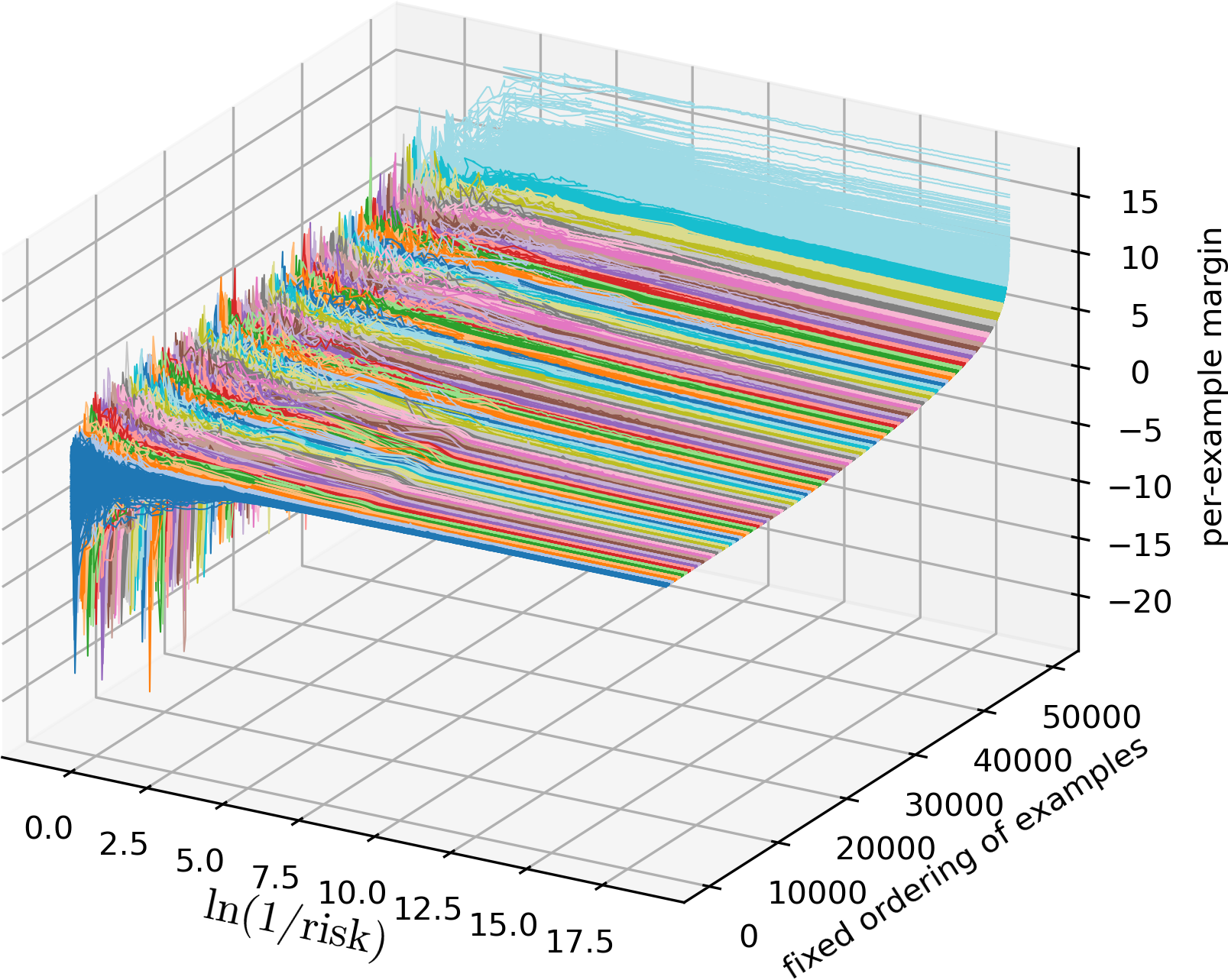}
    \caption{Margins while training H-AlexNet.}
    \label{fig:margins:halexnet}
  \end{subfigure}\hfill
  \begin{subfigure}[t]{0.48\textwidth}
    \centering
\includegraphics[width=\textwidth]{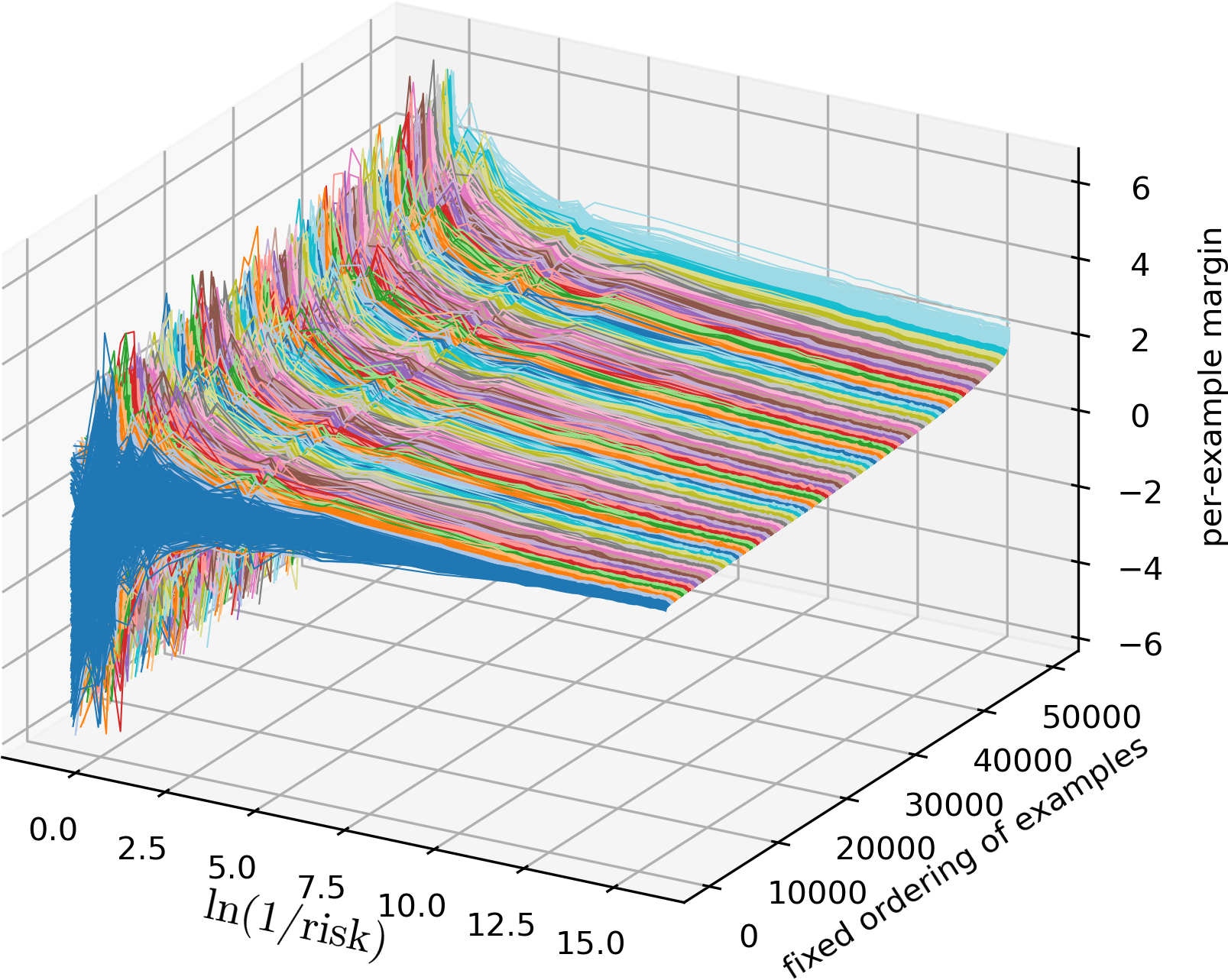}
    \caption{Margins while training DenseNet.}
    \label{fig:margins:densenet}
  \end{subfigure}\caption{The margins of all examples in \cifar, plotted against time,
    or rather optimization accuracy $\ln(n/\cL(W_t))$ to remove the effect
    of step size and other implementation coincidences.
    \Cref{fig:margins:halexnet} shows ``H-AlexNet'', a homogeneous version
    of AlexNet as described in the main text \citep{imagenet_sutskever},
    which is handled by our theory.
    \Cref{fig:margins:densenet} shows a standard
    DenseNet \citep{densenet}, which does not fit the theory in
    this work due to skip connections and biases, but still
    exhibits convergence of margins,
    thus suggesting a tantalizing open problem.
  }
  \label{fig:margins}
\end{figure}

\subsection{Second result: gradient alignment}

Our second contribution, in \Cref{sec:alignment}, is that if the network has
locally Lipschitz gradients, then these gradients also converge, and are
\emph{aligned} to the gradient flow path (cf. \Cref{fact:alignment}).
\begin{center}
  \emph{The gradient flow path, and the gradient of the risk along the path,
  converge to the same direction.}
\end{center}

As a practical consequence of this, recall the use of gradients within the
interpretability literature, specifically in \emph{saliency
maps}~\citep{saliency_maps_sanity_checks}: if gradients do not converge in
direction then saliency maps can change regardless of the number of iterations
used to produce them.
As a theoretical consequence, directional convergence and alignment imply margin
maximization in a variety of situations: this holds in the deep linear case,
strengthening prior work \citep{nati_lnn,align}, and in the 2-homogeneous
network case, with an assumption taken from the infinite width setting
\citep{chizat_bach_imp}, but presented here with finite width.

\subsection{Further related work}\label{sec:rw}

Our analysis is heavily inspired and influenced by the work of
\citet{kaifeng_jian_margin},
who studied margin maximization of homogeneous networks,
establishing monotonicity of a \emph{smoothed margin}, a quantity we also use.
However, they did not prove directional convergence but instead must use
subsequences.
Their work also left open alignment and global margin maximization.

\paragraph{Directional convergence.}

A standard approach to resolve directional convergence and similar questions is
to establish that the objective function in question is \emph{definable in some
o-minimal structure}, which as mentioned before, limits oscillations and other
complicated behavior.
This literature cannot be directly applied to our setting, owing to a combination of
nonsmooth layers like the ReLU and max-pooling, and the exponential function
used in the cross entropy loss, and as a result,
our proofs need to rebuild many o-minimal results from the ground up.

In more detail, an important problem in the o-minimal literature is the
\emph{gradient conjecture} of Ren\'{e} Thom: it asks when the existence of
$\lim_{t\to\infty}W_t=z$ further implies
$\lim_{t\to\infty} \nicefrac{(W_t-z)}{\|W_t-z\|}$ exists, and was established in
various definable scenarios by \citet{kurdyka_thom,kurdyka_quasi} via related
Kurdyka-\L{}ojasiewicz inequalities~\citep{kurdyka_grad}.
The underlying proof ideas can also be used to analyze
$\lim_{t\to\infty} \nicefrac {W_t}{\|W_t\|}$ when the weights go to
infinity~\citep{grandjean_limit}.
However, the prior results require the objective function to be either real
analytic, or definable in a ``polynomially-bounded'' o-minimal structure.
The first case causes the aforementioned nonsmoothness issue,
and excludes many common layers in deep learning such as the ReLU and max-pooling.
The second case excludes the exponential function, and means the logistic
and cross-entropy losses cannot be handled.
To resolve these issues, we had to redo large portions of the o-minimality theory,
such as the nonsmooth unbounded Kurdyka-{\L}ojasiewicz inequalities that can
handle the exponential/logistic loss, as presented in \Cref{sec:dir}.

\paragraph{Alignment.}

As discussed in \Cref{sec:alignment}, alignment implies the gradient flow
reaches a stationary point of the limiting margin maximization objective,
and therefore is related to various statements and results
throughout the literature on implicit bias and margin
maximization \citep{nati_logistic,min_norm}.
This stationary point perspective also appears in some nonlinear works, for
instance in the aforementioned work on margins by \citet{kaifeng_jian_margin},
which showed that \emph{subsequences} of the gradient flow converge to such
stationary points; in addition to fully handling the gradient flow, the present
work also differs in that alignment is in general a stronger notion, in that it
is unclear how to prove alignment as a consequence of convergence to KKT points.
Additionally, alignment can still hold when the objective function is not
definable and directional convergence is false, for example on the homogeneous
Mexican hat function, which cannot be handled by the approach in
\citep[Appendix J]{kaifeng_jian_margin}.
As a final pointer to the literature, many implicit bias works explicitly assume
directional convergence and some version of alignment
\citep{nati_lnn,chizat_bach_imp}, but neither do these works indicate a possible proof,
nor do they provide conclusive evidence.

\subsection{Experimental overview}

The experiments in \Cref{fig:contours,fig:margins} are performed in as standard a way as possible
to highlight that directional convergence is a reliable property;
full details are in \Cref{app_sec:empirical}.
Briefly, \Cref{fig:contours}
uses synthetic data and vanilla gradient descent (no momentum, no weight decay, etc.)
on a 10,000 node wide 2-layer \emph{squared} ReLU network and its Neural Tangent Kernel classifier;
by using the squared ReLU, both our directional convergence and our alignment results apply.
\Cref{fig:margins} uses standard \cifar firstly with a modified homogeneous AlexNet and
secondly with an
unmodified DenseNet, respectively inside and outside our assumptions.
SGD was used on \cifar due to training set size,
and seeing how directional convergence still seems to occur,
suggests another open problem.

\section{Preliminaries and assumptions}\label{sec:prelim}

In this section, we first introduce the notions of Clarke subdifferentials and
o-minimal structures, and then use these notions to describe the network model,
gradient flow, and \Cref{cond:Phi,cond:init}.
Throughout this paper, $\|\cdot\|$ denotes the $\ell_2$ (Frobenius) norm, and
$\|\cdot\|_\sigma$ denotes the spectral norm.

\paragraph{Locally Lipschitz functions and Clarke subdifferentials.}

Consider a function $f:D\to\R$ with $D$ open.
We say that $f$ is \emph{locally Lipschitz} if for any $x\in D$, there exists a
neighborhood $U$ of $x$ such that $f|_U$ is Lipschitz continuous.
We say that $f$ is $C^1$ if $f$ is continuously differentiable on $D$.

If $f$ is locally Lipschitz, it holds that
$f$ is differentiable a.e. \citep[Theorem 9.1.2]{borwein_lewis}.
The \emph{Clarke subdifferential} of $f$ at $x\in D$ is defined as
\begin{align*}
    \partial f(x):=\conv\setdef{\lim_{i\to\infty}\nabla f(x_i)}{x_i\in D,\nabla f(x_i)\textrm{ exists},\lim_{i\to\infty}x_i=x},
\end{align*}
which is nonempty convex compact \citep{clarke}, and if $f$ is continuously
differentiable at $x$, then $\partial f(x)=\{\nabla f(x)\}$.
Vectors in $\partial f(x)$ are called \emph{subgradients}, and we let $\cs f(x)$
denote the unique minimum-norm subgradient:
\begin{align*}
    \cs f(x):=\argmin_{x^*\in\partial f(x)}\|x^*\|.
\end{align*}
In the following analysis, we use $\cs f$ in many places that seem to call on
$\nabla f$.

\paragraph{O-minimal structures and definable functions.}

Formally, an o-minimal structure is a collection $\cS=\{\cS_n\}_{n=1}^\infty$,
where $\cS_n$ is a set of subsets of $\R^n$ which includes all algebraic sets
and is closed under finite union/intersection and complement, Cartesian product,
and projection, and $\cS_1$ consists of finite unions of open intervals and
points.
A set $A\subset\R^n$ is \emph{definable} if $A\in\cS_n$, and a function
$f:D\to\R^m$ with $D\subset\R^n$ is \emph{definable} if its graph is in
$\cS_{n+m}$.
More details are given in \Cref{app_sec:o}.

Many natural functions and operations are definable.
First of all, definability of functions is stable under algebraic operations,
composition, inverse, maximum and minimum, etc.
Moreover, \citet{wilkie_exp} proved that there exists an o-minimal structure
where polynomials and the exponential function are definable.
Consequently, definability allows many common layer types in deep learning, such
as fully-connected/convolutional/ReLU/max-pooling layers, skip connections, the
cross entropy loss, etc.; moreover, they can be composed arbitrarily
As will be discussed later, what is still missing is the handling of the
gradient flow on such functions.

\paragraph{The network model.}

Consider a dataset $\{(x_i,y_i)\}_{i=1}^n$, where $x_i\in\R^d$ are features and
$y_i\in\{-1,+1\}$ are binary labels, and a predictor $\Phi(\cdot;W):\R^d\to\R$
with parameters $W\in\R^\pc$.
We make the following assumption on the predictor $\Phi$.
\begin{assumption}\label{cond:Phi}
    For any fixed $x$, the prediction $W\mapsto\Phi(x;W)$ as a function of
    $W$ is locally Lipschitz, $L$-positively homogeneous for some $L>0$, and
    definable in some o-minimal structure including the exponential function.
\end{assumption}
As mentioned before, homogeneity means that $\Phi(x;cW) = c^L \Phi(x;W)$ for any
$c \geq 0$.
This means, for instance, that linear, convolutional, ReLU, and max-pooling
layers are permitted, but not skip connections and biases.
Homogeneity is used heavily throughout the theoretical study of deep networks
\citep{kaifeng_jian_margin}.

Given a decreasing loss function $\ell$, the total loss (or \emph{unnormalized empirical risk})
is given by
\begin{align*}
    \cL(W):=\sum_{i=1}^{n}\ell\del{y_i\Phi(x_i;W)}=\sum_{i=1}^{n}\ell(p_i(W)),
\end{align*}
where $p_i(W):=y_i\Phi(x_i;W)$ are also locally Lipschitz, $L$-positively
homogeneous and definable under \Cref{cond:Phi}.
We consider the exponential loss $\lexp(z):=e^{-z}$ and the logistic loss
$\llog(z):=\ln(1+e^{-z})$, in which case $\cL$ is also locally Lipschitz and
definable.

\paragraph{Gradient flow.}

As in \citep{davis_tame,kaifeng_jian_margin}, a curve $z$ from an interval $I$
to some real space $\R^m$ is called an \emph{arc} if it is absolutely continuous
on any compact subinterval of $I$.
It holds that an arc is a.e. differentiable, and the composition of an arc and a
locally Lipschitz function is still an arc.
We consider a gradient flow
$W:[0,\infty)\to\R^\pc$ that is an arc and satisfies
\begin{align}\label{eq:gf}
    \frac{\dif W_t}{\dif t}\in-\partial\cL(W_t),\quad\textrm{for a.e. }t\ge0.
\end{align}
Our second assumption is on the initial risk,
and appears in prior work \citep{kaifeng_jian_margin}.
\begin{assumption}\label{cond:init}
    The initial iterate $W_0$ satisfies $\cL(W_0)<\ell(0)$.
\end{assumption}
As mentioned before, this assumption encapsulates our focus on the
``late training'' phase;
some other analysis, for instance the neural tangent kernel, can be first
applied to ensure $\cL(W_0)<\ell(0)$.

\section{Directional convergence}\label{sec:dir}

We now turn to stating our main result on directional convergence and sketching
its analysis.
As \Cref{cond:Phi,cond:init} imply $\|W_t\|\to\infty$
\citep{kaifeng_jian_margin}, we study the normalized flow $\tW_t:=W_t/\|W_t\|$,
whose convergence is a formal way of studying the directional convergence of
$W_t$.
As mentioned before, directional convergence is false in general
\citep{kaifeng_jian_margin}, but definability suffices to ensure it.
Throughout, for general nonzero $W$, we will use $\tW := W/\|W\|$.

\begin{theorem}\label{fact:dir}
    Under \Cref{cond:Phi,cond:init}, for $\lexp$ and $\llog$, the curve swept by
    $\tW_t$ has finite length, and thus $\tW_t$ converges.
\end{theorem}

A direct consequence of \Cref{fact:dir} is the convergence of the \emph{margin
distribution} (i.e., normalized outputs).
Due to homogeneity, for any nonzero $W$, we have $p_i(W)/\|W\|^L=p_i(\tW)$, and
thus the next result follows from \Cref{fact:dir}.
\begin{corollary}\label{fact:margin_conv}
    Under \Cref{cond:Phi,cond:init}, for $\lexp$ and $\llog$, it holds that
    $p_i(W_t)/\|W_t\|^L$ converges for all $1\le i\le n$.
\end{corollary}

Next we give a proof sketch of \Cref{fact:dir}; the full proofs of the
Kurdyka-\L{}ojasiewicz inequalities
(\Cref{fact:unbounded_loja_1,fact:unbounded_loja_2}) are given in
\Cref{app_sec:loja}, while the other proofs are given in \Cref{app_sec:dir}.

\subsection{A proof sketch of \Cref{fact:dir}}

The \emph{smoothed margin} introduced in \citep{kaifeng_jian_margin} is crucial
in our analysis: given $W\ne0$, let
\begin{align*}
    \alpha(W):=\ell^{-1}\del{\cL(W)},\quad\textrm{and}\quad\ta(W):=\frac{\alpha(W)}{\|W\|^L}.
\end{align*}
For simplicity, let $\ta_t$ denote $\ta(W_t)$, and $\zeta_t$ denote the length
of the path swept by $\tW_t = W_t/\|W_t\|$ from time $0$ to $t$.
\citet{kaifeng_jian_margin} proved that $\ta_t$ is nondecreasing with some limit
$a\in(0,\infty)$, and $\|W_t\|\to\infty$.
We invoke a standard but sophisticated tool from the definability literature to
aid in proving $\zeta_t$ is finite: formally, a function $\Psi:[0,\nu)\to\R$ is
called a \emph{desingularizing function} when $\Psi$ is continuous on $[0,\nu)$
with $\Psi(0)=0$, and continuously differentiable on $(0,\nu)$ with $\Psi'>0$;
in words, a desingularizing function is a \emph{witness} to the fact that
the flow is asymptotically well-behaved.
As we will sketch after stating the \namecref{fact:ta_zeta_ratio},
this immediately leads to a proof of \Cref{fact:dir}.
\begin{lemma}\label{fact:ta_zeta_ratio}
    There exist $R>0$, $\nu>0$ and a definable desingularizing function $\Psi$
    on $[0,\nu)$, such that for a.e. large enough $t$ with $\|W_t\|>R$ and
    $\ta_t>a-\nu$, it holds that
    \begin{align*}
        \frac{\dif\zeta_t}{\dif t}\le-c\frac{\dif\Psi\del{a-\ta_t}}{\dif t}
    \end{align*}
    for some constant $c>0$.
\end{lemma}

To prove \Cref{fact:dir} from here, let $t_0$ be large enough so that the
conditions of \Cref{fact:ta_zeta_ratio} hold for all $t\ge t_0$: then we have
$\lim_{t\to\infty}\zeta_t\le\zeta_{t_0}+c\Psi\del{a-\ta_{t_0}}<\infty$, and thus
the path length is finite.

Below we sketch the proof of \Cref{fact:ta_zeta_ratio}, which is based on a
careful comparison of $\dif\ta_t/\dif t$ and $\dif\zeta_t/\dif t$.
The proof might be hard to parse due to the extensive use of $\cs$, the
minimum-norm Clarke subgradient; at first reading, the condition of local
Lipschitz continuity can just be replaced with continuous differentiability, in
which case the Clarke subgradient is just the normal gradient.

Given any function $f$ which is locally Lipschitz around a nonzero $W$, let
\begin{align*}
    \csr f(W):=\ip{\cs f(W)}{\tW}\tW\quad\textrm{and}\quad\csp f(W):=\cs f(W)-\csr f(W)
\end{align*}
denote the radial and spherical parts of $\cs f(W)$ respectively.
First note the following technical characterization of $\nicefrac{\dif\ta_t}{\dif t}$
and $\nicefrac{\dif\zeta_t}{\dif t}$ using the radial and spherical components
of relevant Clarke subgradients.
\begin{lemma}\label{fact:ta_zeta}
    It holds for a.e. $t\ge0$ that
    \begin{align*}
        \frac{\dif\ta_t}{\dif t}=\enVert{\csr\ta(W_t)}\enVert{\csr\cL(W_t)}+\enVert{\csp\ta(W_t)}\enVert{\csp\cL(W_t)},\ \ \textrm{and}\ \ \frac{\dif\zeta_t}{\dif t}=\frac{\enVert{\csp\cL(W_t)}}{\|W_t\|}.
    \end{align*}
\end{lemma}

For simplicity, in the discussion here we consider the case that all subgradients in
\Cref{fact:ta_zeta} are nonzero, with the general case handled in the full proofs
in the appendices.
Then \Cref{fact:ta_zeta} implies
\begin{align}\label{eq:ta_zeta_ratio}
    \frac{\dif\ta_t}{\dif\zeta_t}=\frac{\dif\ta_t/\dif t}{\dif\zeta_t/\dif t}=\|W_t\|\del{\frac{\enVert{\csr\cL(W_t)}}{\enVert{\csp\cL(W_t)}}\enVert{\csr\ta(W_t)}+\enVert{\csp\ta(W_t)}}.
\end{align}
As in \citep{kurdyka_quasi,grandjean_limit}, to bound \cref{eq:ta_zeta_ratio},
we further consider two cases depending on the ratio
$\enVert{\csp\ta(W_t)}/\enVert{\csr\ta(W_t)}$.

If $\enVert{\csp\ta(W_t)}/\enVert{\csr\ta(W_t)}\ge c_1\|W_t\|^{L/3}$ for some
constant $c_1>0$, then \Cref{fact:ta_zeta_ratio} follows from
$\dif\ta_t/\dif\zeta_t\ge\|W_t\|\enVert{\csp\ta(W_t)}$ as given by
\cref{eq:ta_zeta_ratio}, and the following Kurdyka-\L{}ojasiewicz inequality.
Its proof is based on the proof idea of \citep[Proposition 6.3]{kurdyka_quasi},
but further handles the unbounded and nonsmooth setting.
\begin{lemma}\label{fact:unbounded_loja_1}
    Given a locally Lipschitz definable function $f$ with an open domain
    $D\subset\setdef{x}{\|x\|>1}$, for any $c,\eta>0$, there exists $\nu>0$ and
    a definable desingularizing function $\Psi$ on $[0,\nu)$ such that
    \begin{align*}
        \Psi'\del{f(x)}\|x\|\enVert{\cs f(x)}\ge1,\quad\textrm{if }f(x)\in(0,\nu)\textrm{ and }\enVert{\csp f(x)}\ge c\|x\|^\eta\enVert{\csr f(x)}.
    \end{align*}
\end{lemma}

On the other hand, if
$\enVert{\csp\ta(W_t)}/\enVert{\csr\ta(W_t)}\le c_1\|W_t\|^{L/3}$, then a
careful calculation (using \Cref{fact:alpha_cs,fact:ta_cs,fact:margins}) can
show that for some constants $c_2,c_3>0$,
\begin{align*}
    \frac{\enVert{\csr\cL(W_t)}}{\enVert{\csp\cL(W_t)}}\ge c_2\|W_t\|^{2L/3},\quad\textrm{and}\quad \frac{\enVert{\csr\ta(W_t)}}{\enVert{\cs\ta(W_t)}}\ge c_3\|W_t\|^{-L/3}.
\end{align*}
It then follows from \cref{eq:ta_zeta_ratio} that
$\dif\ta_t/\dif\zeta_t\ge c_2c_3\|W_t\|^{4L/3}\enVert{\cs\ta(W_t)}$.
In this case we give the following Kurdyka-\L{}ojasiewicz inequality, which
implies \Cref{fact:ta_zeta_ratio}.
\begin{lemma}\label{fact:unbounded_loja_2}
    Given a locally Lipschitz definable function $f$ with an open domain
    $D\subset\setdef{x}{\|x\|>1}$, for any $\lambda>0$, there exists $\nu>0$ and
    a definable desingularizing function $\Psi$ on $[0,\nu)$ such that
    \begin{align*}
        \max\cbr{1,\frac{2}{\lambda}}\Psi'\del{f(x)}\|x\|^{1+\lambda}\enVert{\cs f(x)}\ge1,\quad\textrm{if }f(x)\in(0,\nu).
    \end{align*}
\end{lemma}

\section{Alignment between the gradient flow path and gradients}\label{sec:alignment}

\Cref{fact:dir} gave our directional convergence result,
namely that the normalized iterate $W_t/\|W_t\|$ converges to some direction.
Next we show and discuss our alignment result,
that if all $p_i$ have locally Lipschitz gradients, then along the
gradient flow path, $-\nL(W_t)$ converges to the same direction as $W_t$.

\begin{theorem}\label{fact:alignment}
    Under \Cref{cond:Phi,cond:init}, if all $p_i$ further have locally Lipschitz
    gradients, then $-\nL(W_t)$ and $W_t$ converge to the same direction,
    meaning the angle between $W_t$ and $-\nL(W_t)$ converges to zero.
    If all $p_i$ are twice continuously differentiable, then the same result
    holds without the definability condition (cf. \Cref{cond:Phi}).
\end{theorem}

Below we first sketch the proof of \Cref{fact:alignment}, with full details
in \Cref{app_sec:alignment}, and then in \Cref{sec:margins}
present a few global margin maximization consequences, which are proved in
\Cref{app_sec:margins}.

\subsection{A proof sketch of \Cref{fact:alignment}}

Recall that $\lim_{t\to\infty}\alpha(W_t)/\|W_t\|^L=a$.
The first observation is that $\alpha(W_t)$, the smoothed margin function,
asymptotes to the exact margin $\min_{1\le i\le n}p_i(W_t)$ which is
$L$-positively homogeneous.
Therefore $\alpha$ is asymptotically $L$-positively homogeneous, and formally we
can show
\begin{align}\label{eq:beta_limit}
    \lim_{t\to\infty}\ip{\frac{\na(W_t)}{\|W_t\|^{L-1}}}{\frac{W_t}{\|W_t\|}}=\lim_{t\to\infty}\frac{\ip{\na(W_t)}{W_t}}{\|W_t\|^L}=aL,
\end{align}
which can be viewed as an asymptotic version of Euler's homogeneous function
theorem (cf. \Cref{fact:clarke_euler}).
Consequently, the inner product between $\na(W_t)/\|W_t\|^{L-1}$ and $\tW_t$
converges.

Let $\theta_t$ denote the angle between $W_t$ and $-\nL(W_t)$, which is also the
angle between $W_t$ and $\na(W_t)$, since $\nL(W_t)$ and $\na(W_t)$ point to
opposite directions by the chain rule.
By \citep[Corollary C.10]{kaifeng_jian_margin}, given any $\epsilon>0$,
there exists a time $t_\epsilon$ such that $\theta_{t_\epsilon}<\epsilon$.
The question is whether such a small angle can be maintained after $t_\epsilon$.
This is not obvious since, as mentioned above, the smoothed margin $\alpha(W_t)$
asymptotes to the exact margin $\min_{1\le i\le n}p_i(W_t)$, which may be
nondifferentiable even with smooth $p_i$, due to nondifferentiability of the
minimum.
Consequently, the exact margin may have discontinuous Clarke subdifferentials,
and since the smoothed margin asymptotes to it, it is unclear whether
$\theta_t\to0$.  (This point was foreshadowed earlier, where it was pointed
out that alignment is not a clear consequence of convergence to stationary points of the
margin maximization objective.)

To handle this, the key to our analysis is the potential function
$\cJ(W):=\enVert{\na(W_t)}^2/\|W_t\|^{2L-2}$.
Suppose at time $t$, it holds that  $\ip{\na(W_t)/\|W_t\|^{L-1}} {\tW_t}$ is
close to $aL$, and $\theta_t$ is very small.
If $\theta_{t'}$ becomes large again at some $t'>t$, it must follows that
$\cJ(W_{t'})$ is much larger than $\cJ(W_t)$.
We prove that this is impossible, by showing that
\begin{align}\label{eq:jt_int_limit}
    \lim_{t\to\infty}\int_t^{\infty}\frac{\dif\cJ(W_\tau)}{\dif\tau}\dif\tau=0,
\end{align}
and thus \Cref{fact:alignment} follows.
The proof of \cref{eq:jt_int_limit} is motivated by the dual convergence
analysis in \citep{dual_conv}, and also uses the positive homogeneity of
$\nabla p_i$ and $\nabla^2p_i$ (which exist a.e.).

\subsection{Main alignment consequence: margin maximization}
\label{sec:margins}

A variety of (global) margin maximization results are immediate consequences of
directional convergence and alignment.
This subsection investigates two examples:
deep linear networks, and shallow squared ReLU networks.

Deep linear networks predict with $\Phi(x_i;W) = A_L\cdots A_1x_i$, where the
parameters $W = (A_L,\ldots, A_1)$ are organized into $L$ matrices.
This setting has been considered in the literature, but the original work
assumed directional convergence, alignment and a condition on the support
vectors \citep{nati_lnn}; a follow-up dropped the directional convergence and
alignment assumptions, but instead assumed the support vectors span the
space $\R^d$ \citep{align}.
As follows, we not only drop the all aforementioned assumptions, but moreover
include a \emph{proof} rather than an assumption of directional convergence.

\begin{proposition}
  \label{fact:deep_linear}
  Suppose $W_t = (A_L(t),\ldots,A_1(t))$ and $\cL(W_0) < \ell(0)$.
  Then a unique linear max margin predictor
  $\baru := \argmax_{\|u\|\leq 1} \min_i y_i x_i^\T u$ exists, and there exist
  unit vectors $(v_L,\ldots,v_1,v_0)$ with $v_L = 1$ and $v_0 = \baru$ such that
  \[
    \lim_{t\to\infty} \frac{A_j(t)}{\|A_j(t)\|} = v_j v_{j-1}^\T
    \qquad
    \text{and}
    \qquad
    \lim_{t\to\infty} \frac {A_L(t)\cdots A_1(t)}{\|A_L(t) \cdots A_1(t)\|} = \baru^\T.
\]
\end{proposition}

Thanks to directional convergence and alignment
(cf. \Cref{fact:dir,fact:alignment}), the proof boils down to writing
down the gradient expression for each layer and doing some algebra.

A more interesting example is a certain 2-homogeneous case,
which despite its simplicity is a universal approximator;
this setting was studied by
\citet{chizat_bach_imp}, who considered the infinite width case,
and established margin maximization under \emph{assumptions} of directional
convergence and gradient convergence.
Unfortunately, it is not clear if \Cref{fact:dir,fact:alignment} can be applied to fill
these assumptions, since they do not handle infinite width, and indeed it is not clear
if infinite width networks or close relatives are definable in an o-minimal structure.
Instead, here we consider the finite width case, albeit with an additional assumption.

Following \citep[S-ReLU]{chizat_bach_imp},
organize $W_t$ into $m$ rows $(w_j(t))_{j=1}^m$,
with normalizations $\theta_j(t) := w_j(t)/\|w_j(t)\|$ where $\theta_j(t) = 0$ when $\|w_j(t)\|=0$,
and consider
\begin{equation}
  \Phi(x_i; W) := \sum_j (-1)^j \max\{0,w_j^\T x_i\}^2
  \quad
  \text{and}
  \quad
  \varphi_{ij}(w) := y_i (-1)^j \max\{0,w^\T x_i\}^2,
  \label{eq:2homo}
\end{equation}
whereby $p_i(W) = \sum_j \varphi_{ij}(w_j)$, and $\Phi$, $p_i$,
and $\varphi_{ij}$ are all 2-homogeneous and definable.
(The ``$(-1)^j$'' may seem odd, but is an easy trick to get universal approximation without
outer weights.)

\begin{proposition}
  \label{fact:covering}
  Consider the setting in \cref{eq:2homo} along with
  $\cL(W_0) < \ell(0)$
  and $\|x_i\|\leq1$.
\begin{enumerate}
    \item \textbf{(Local guarantee.)}
      $s\in\R^m$ with $s_j(t) := \nicefrac{\|w_j(t)\|^2}{\|W_t\|^2}$
      satisfies $s \to \barp\in\Delta_m$ (probability simplex on $m$ vertices),
      and $\theta_j \to \btheta_j$ with $\btheta_j=0$ if $s_j = 0$, and
      \[
        a
        = \lim_{t\to\infty} \min_i \frac{p_i(W_t)}{\|W_t\|^2}
        = \lim_{t\to\infty} \min_i \sum_j s_j(t) \varphi_{ij}(\theta_j(t))
        = \min_i \max_{s\in \Delta_m} \sum_j s_j\varphi_{ij}(\btheta_j).
      \]

    \item \textbf{(Global guarantee.)}
      Suppose the \emph{covering condition}:
      there exist $t_0$ and $\eps>0$ with
      \[
        \max_j \|\theta_j(t_0) - \btheta_j\|_2 \leq \eps,
        \ \text{and}\ {}
        \max_{\theta'\in\S^{d-1}}
        \max\cbr[2]{
          \min_{2\mid j} \|\theta_j(t_0) - \theta'\|,
          \min_{2\nmid j} \|\theta_j(t_0) -\theta'\|
        }\leq \eps,
      \]
      where $\S^{d-1} := \{ \theta \in\R^d : \|\theta\|=1\}$.
      Then margins are approximately (globally) maximized:
      \[
        \lim_{t\to\infty} \min_i \frac{p_i(W_t)}{\|W_t\|^2}
        \geq
        \max_{\nu\in\cP(\S^{d-1})}
        \min_i
        y_i \int\max\{0, x_i^\T \theta\}^2\dif\nu(\theta) - 4\eps,
      \]
      where $\cP(\S^{d-1})$ is the set of \emph{signed} measures on $\S^{d-1}$ with mass at most $1$.
    \end{enumerate}
\end{proposition}

The first part (the ``local guarantee'')
characterizes the limiting margin as the maximum margin of a \emph{linear}
problem obtained by taking the limiting directions $(\btheta_j)_{j=1}^m$ and treating
the resulting $\varphi_{ij}(\btheta_j)$ as features.
The quality of this margin is bad if the limiting directions are bad, and therefore
we secondly (the ``global guarantee'')
consider a case where our margin is nearly as good as the
\emph{infinite width global max margin value} as defined by
\citep[eq. (5)]{chizat_bach_imp}; see discussion therein for a justification of this choice,
and moreover calling it the globally maximal margin.

The \emph{covering condition} deserves further discussion.
In the infinite width setting, it holds for all $\eps>0$ assuming directional
convergence \citep[Proof of Theorem D.1]{chizat_bach_imp}, but cannot hold in
such generality here as we are dealing with finite width.
Similar properties have appeared throughout the literature:
\citet[Section 3]{wei_reg} explicitly re-initialized network nodes to guarantee a good
covering,
and more generally \citep{rong_saddle_point_escape} added noise to escape
saddle points in general optimization problems.

\section{Concluding remarks and open problems}

In this paper, we established that the normalized parameter vectors
$\nicefrac {W_t}{\|W_t\|}$ converge, and that under an additional assumption of
locally Lipschitz gradients, the gradients also converge and align with the
parameters.

There are many promising avenues for future work based on these results.
One basic line is to weaken our assumptions: dropping homogeneity to allow for
DenseNet and ResNet, and analyzing finite-time methods like (stochastic)
gradient descent, and moreover their rates of convergence.
We also handled only the binary classification case, however our tools should
directly allow for cross-entropy.

Another direction is into further global margin maximization results, beyond the
simple networks in \Cref{sec:margins}, and into related generalization
consequences of directional convergence and alignment.

\subsection*{Acknowledgements}

The authors thank Zhiyuan Li and Kaifeng Lyu for lively discussions during an early phase
of the project.
The authors are grateful for support from the NSF under grant IIS-1750051,
and from NVIDIA via a GPU grant.

\bibliography{bib}

\begin{thebibliography}{41}
\providecommand{\natexlab}[1]{#1}
\providecommand{\url}[1]{\texttt{#1}}
\expandafter\ifx\csname urlstyle\endcsname\relax
  \providecommand{\doi}[1]{doi: #1}\else
  \providecommand{\doi}{doi: \begingroup \urlstyle{rm}\Url}\fi

\bibitem[Adebayo et~al.(2018)Adebayo, Gilmer, Muelly, Goodfellow, Hardt, and
  Kim]{saliency_maps_sanity_checks}
Julius Adebayo, Justin Gilmer, Michael Muelly, Ian Goodfellow, Moritz Hardt,
  and Been Kim.
\newblock Sanity checks for saliency maps.
\newblock In \emph{NIPS}, 2018.
\newblock {\tt arXiv:1810.03292 [cs.CV]}.

\bibitem[Allen-Zhu et~al.(2018)Allen-Zhu, Li, and Song]{allen_deep_opt}
Zeyuan Allen-Zhu, Yuanzhi Li, and Zhao Song.
\newblock A convergence theory for deep learning via over-parameterization.
\newblock \emph{arXiv preprint arXiv:1811.03962}, 2018.

\bibitem[Bartlett et~al.(2017)Bartlett, Foster, and Telgarsky]{spec}
Peter~L Bartlett, Dylan~J Foster, and Matus~J Telgarsky.
\newblock Spectrally-normalized margin bounds for neural networks.
\newblock In \emph{Advances in Neural Information Processing Systems}, pages
  6240--6249, 2017.

\bibitem[Bolte et~al.(2007)Bolte, Daniilidis, Lewis, and Shiota]{bolte_clarke}
J{\'e}r{\^o}me Bolte, Aris Daniilidis, Adrian Lewis, and Masahiro Shiota.
\newblock Clarke subgradients of stratifiable functions.
\newblock \emph{SIAM Journal on Optimization}, 18\penalty0 (2):\penalty0
  556--572, 2007.

\bibitem[Borwein and Lewis(2000)]{borwein_lewis}
Jonathan Borwein and Adrian Lewis.
\newblock \emph{Convex Analysis and Nonlinear Optimization}.
\newblock Springer Publishing Company, Incorporated, 2000.

\bibitem[{Chizat} and {Bach}(2018)]{chizat_bach_meanfield}
L{\' e}na{\" i}c {Chizat} and Francis {Bach}.
\newblock On the global convergence of gradient descent for over-parameterized
  models using optimal transport.
\newblock In \emph{NIPS}, 2018.
\newblock {\tt arXiv:1805.09545 [math.OC]}.

\bibitem[Chizat and Bach(2020)]{chizat_bach_imp}
Lenaic Chizat and Francis Bach.
\newblock Implicit bias of gradient descent for wide two-layer neural networks
  trained with the logistic loss.
\newblock \emph{arXiv preprint arXiv:2002.04486}, 2020.

\bibitem[Clarke(1975)]{clarke}
Frank~H Clarke.
\newblock Generalized gradients and applications.
\newblock \emph{Transactions of the American Mathematical Society},
  205:\penalty0 247--262, 1975.

\bibitem[Clarke(1983)]{clarke_opt}
Frank~H. Clarke.
\newblock \emph{Optimization and Nonsmooth Analysis}.
\newblock Siam Classics in Applied Mathematics, 1983.

\bibitem[Coste(2000)]{coste_o}
Michel Coste.
\newblock \emph{An introduction to o-minimal geometry}.
\newblock Istituti editoriali e poligrafici internazionali Pisa, 2000.

\bibitem[Davis et~al.(2020)Davis, Drusvyatskiy, Kakade, and Lee]{davis_tame}
Damek Davis, Dmitriy Drusvyatskiy, Sham Kakade, and Jason~D Lee.
\newblock Stochastic subgradient method converges on tame functions.
\newblock \emph{Foundations of computational mathematics}, 20\penalty0
  (1):\penalty0 119--154, 2020.

\bibitem[Du et~al.(2018)Du, Lee, Li, Wang, and Zhai]{du_deep_opt}
Simon~S Du, Jason~D Lee, Haochuan Li, Liwei Wang, and Xiyu Zhai.
\newblock Gradient descent finds global minima of deep neural networks.
\newblock \emph{arXiv preprint arXiv:1811.03804}, 2018.

\bibitem[Freund and Schapire(1997)]{freund_schapire_adaboost}
Yoav Freund and Robert~E. Schapire.
\newblock A decision-theoretic generalization of on-line learning and an
  application to boosting.
\newblock \emph{J. Comput. Syst. Sci.}, 55\penalty0 (1):\penalty0 119--139,
  1997.

\bibitem[Ge et~al.(2015)Ge, Huang, Jin, and Yuan]{rong_saddle_point_escape}
Rong Ge, Furong Huang, Chi Jin, and Yang Yuan.
\newblock Escaping from saddle points --- online stochastic gradient for tensor
  decomposition.
\newblock In \emph{COLT}, 2015.
\newblock {\tt arXiv:1503.02101 [cs.LG]}.

\bibitem[Grandjean(2007)]{grandjean_limit}
V~Grandjean.
\newblock On the limit set at infinity of a gradient trajectory of a
  semialgebraic function.
\newblock \emph{Journal of Differential Equations}, 233\penalty0 (1):\penalty0
  22--41, 2007.

\bibitem[Gunasekar et~al.(2018{\natexlab{a}})Gunasekar, Lee, Soudry, and
  Srebro]{GLSS18}
Suriya Gunasekar, Jason Lee, Daniel Soudry, and Nathan Srebro.
\newblock Characterizing implicit bias in terms of optimization geometry.
\newblock \emph{arXiv preprint arXiv:1802.08246}, 2018{\natexlab{a}}.

\bibitem[Gunasekar et~al.(2018{\natexlab{b}})Gunasekar, Lee, Soudry, and
  Srebro]{nati_lnn}
Suriya Gunasekar, Jason~D Lee, Daniel Soudry, and Nati Srebro.
\newblock Implicit bias of gradient descent on linear convolutional networks.
\newblock In \emph{Advances in Neural Information Processing Systems}, pages
  9461--9471, 2018{\natexlab{b}}.

\bibitem[Huang et~al.(2017)Huang, Liu, van~der Maaten, and
  Weinberger]{densenet}
Gao Huang, Zhuang Liu, Laurens van~der Maaten, and Kilian~Q. Weinberger.
\newblock Densely connected convolutional networks.
\newblock In \emph{CVPR}, 2017.
\newblock {\tt arXiv:1608.06993v5 [cs.CV]}.

\bibitem[Jacot et~al.(2018)Jacot, Gabriel, and Hongler]{jacot_ntk}
Arthur Jacot, Franck Gabriel, and Cl{\'e}ment Hongler.
\newblock Neural tangent kernel: Convergence and generalization in neural
  networks.
\newblock In \emph{Advances in neural information processing systems}, pages
  8571--8580, 2018.

\bibitem[Ji and Telgarsky(2018{\natexlab{a}})]{align}
Ziwei Ji and Matus Telgarsky.
\newblock Gradient descent aligns the layers of deep linear networks.
\newblock \emph{arXiv preprint arXiv:1810.02032}, 2018{\natexlab{a}}.

\bibitem[Ji and Telgarsky(2018{\natexlab{b}})]{min_norm}
Ziwei Ji and Matus Telgarsky.
\newblock Risk and parameter convergence of logistic regression.
\newblock \emph{arXiv preprint arXiv:1803.07300v2}, 2018{\natexlab{b}}.

\bibitem[Ji and Telgarsky(2019)]{dual_conv}
Ziwei Ji and Matus Telgarsky.
\newblock A refined primal-dual analysis of the implicit bias.
\newblock \emph{arXiv preprint arXiv:1906.04540}, 2019.

\bibitem[Jiang et~al.(2019)Jiang, Krishnan, Mobahi, and
  Bengio]{margindist_predict}
Yiding Jiang, Dilip Krishnan, Hossein Mobahi, and Samy Bengio.
\newblock Predicting the generalization gap in deep networks with margin
  distributions.
\newblock In \emph{ICLR}, 2019.
\newblock {\tt arXiv:1810.00113 [stat.ML]}.

\bibitem[Jiang et~al.(2020)Jiang, Neyshabur, Mobahi, Krishnan, and
  Bengio]{fantastic}
Yiding Jiang, Behnam Neyshabur, Hossein Mobahi, Dilip Krishnan, and Samy
  Bengio.
\newblock Fantastic generalization measures and where to find them.
\newblock In \emph{ICLR}, 2020.
\newblock {\tt arXiv:1912.02178 [cs.LG]}.

\bibitem[Krizhevsky(2009)]{cifar}
Alex Krizhevsky.
\newblock Learning multiple layers of features from tiny images.
\newblock \url{https://www.cs.toronto.edu/~kriz/learning-features-2009-TR.pdf},
  2009.

\bibitem[Krizhevsky et~al.(2012)Krizhevsky, Sutskever, and
  Hinton]{imagenet_sutskever}
Alex Krizhevsky, Ilya Sutskever, and Geoffery Hinton.
\newblock Imagenet classification with deep convolutional neural networks.
\newblock In \emph{NIPS}, 2012.

\bibitem[Kurdyka(1998)]{kurdyka_grad}
Krzysztof Kurdyka.
\newblock On gradients of functions definable in o-minimal structures.
\newblock In \emph{Annales de l'institut Fourier}, volume~48, pages 769--783,
  1998.

\bibitem[Kurdyka et~al.(2000{\natexlab{a}})Kurdyka, Mostowski, and
  Parusinski]{kurdyka_thom}
Krzysztof Kurdyka, Tadeusz Mostowski, and Adam Parusinski.
\newblock Proof of the gradient conjecture of r. thom.
\newblock \emph{Annals of Mathematics}, 152\penalty0 (3):\penalty0 763--792,
  2000{\natexlab{a}}.

\bibitem[Kurdyka et~al.(2000{\natexlab{b}})Kurdyka, Orro, and
  Simon]{kurdyka_sard}
Krzysztof Kurdyka, Patrice Orro, and St{\'e}phane Simon.
\newblock Semialgebraic sard theorem for generalized critical values.
\newblock \emph{Journal of differential geometry}, 56\penalty0 (1):\penalty0
  67--92, 2000{\natexlab{b}}.

\bibitem[Kurdyka et~al.(2006)Kurdyka, Parusi{\'n}ski, et~al.]{kurdyka_quasi}
Krzysztof Kurdyka, Adam Parusi{\'n}ski, et~al.
\newblock Quasi-convex decomposition in o-minimal structures. application to
  the gradient conjecture.
\newblock In \emph{Singularity theory and its applications}, pages 137--177.
  Mathematical Society of Japan, 2006.

\bibitem[L{\^e}~Loi(2010)]{le_lec}
Ta~L{\^e}~Loi.
\newblock Lecture 1: O-minimal structures.
\newblock In \emph{The Japanese-Australian Workshop on Real and Complex
  Singularities: JARCS III}, pages 19--30. Centre for Mathematics and its
  Applications, Mathematical Sciences Institute, The Australian National
  University, 2010.

\bibitem[Lyu and Li(2019)]{kaifeng_jian_margin}
Kaifeng Lyu and Jian Li.
\newblock Gradient descent maximizes the margin of homogeneous neural networks.
\newblock \emph{arXiv preprint arXiv:1906.05890}, 2019.

\bibitem[Mei et~al.(2019)Mei, Misiakiewicz, and
  Montanari]{montanari_mean_field}
Song Mei, Theodor Misiakiewicz, and Andrea Montanari.
\newblock Mean-field theory of two-layers neural networks: dimension-free
  bounds and kernel limit.
\newblock 2019.
\newblock {\tt arXiv:1902.06015 [stat.ML]}.

\bibitem[N{\'e}methi and Zaharia(1992)]{nemethi_milnor}
Andr{\'a}s N{\'e}methi and Alexandru Zaharia.
\newblock Milnor fibration at infinity.
\newblock \emph{Indagationes Mathematicae}, 3\penalty0 (3):\penalty0 323--335,
  1992.

\bibitem[Paszke et~al.(2019)Paszke, Gross, Massa, Lerer, Bradbury, Chanan,
  Killeen, Lin, Gimelshein, Antiga, Desmaison, Kopf, Yang, DeVito, Raison,
  Tejani, Chilamkurthy, Steiner, Fang, Bai, and
  Chintala]{pytorch_official_citation}
Adam Paszke, Sam Gross, Francisco Massa, Adam Lerer, James Bradbury, Gregory
  Chanan, Trevor Killeen, Zeming Lin, Natalia Gimelshein, Luca Antiga, Alban
  Desmaison, Andreas Kopf, Edward Yang, Zachary DeVito, Martin Raison, Alykhan
  Tejani, Sasank Chilamkurthy, Benoit Steiner, Lu~Fang, Junjie Bai, and Soumith
  Chintala.
\newblock Pytorch: An imperative style, high-performance deep learning library.
\newblock In \emph{NeuRIPS}. 2019.

\bibitem[Shallue et~al.(2018)Shallue, Lee, Antognini, Sohl-Dickstein, Frostig,
  and Dahl]{jascha_step_sizes}
Christopher~J. Shallue, Jaehoon Lee, Joseph Antognini, Jascha Sohl-Dickstein,
  Roy Frostig, and George~E. Dahl.
\newblock Measuring the effects of data parallelism on neural network training.
\newblock 2018.
\newblock {\tt arXiv:1811.03600 [cs.LG]}.

\bibitem[Soudry et~al.(2017)Soudry, Hoffer, Nacson, Gunasekar, and
  Srebro]{nati_logistic}
Daniel Soudry, Elad Hoffer, Mor~Shpigel Nacson, Suriya Gunasekar, and Nathan
  Srebro.
\newblock The implicit bias of gradient descent on separable data.
\newblock \emph{arXiv preprint arXiv:1710.10345}, 2017.

\bibitem[Van~den Dries and Miller(1996)]{van_geo}
Lou Van~den Dries and Chris Miller.
\newblock Geometric categories and o-minimal structures.
\newblock \emph{Duke Math. J}, 84\penalty0 (2):\penalty0 497--540, 1996.

\bibitem[Wei et~al.(2018)Wei, Lee, Liu, and Ma]{wei_reg}
Colin Wei, Jason~D Lee, Qiang Liu, and Tengyu Ma.
\newblock Regularization matters: Generalization and optimization of neural
  nets vs their induced kernel.
\newblock \emph{arXiv preprint arXiv:1810.05369}, 2018.

\bibitem[Wilkie(1996)]{wilkie_exp}
Alex~J Wilkie.
\newblock Model completeness results for expansions of the ordered field of
  real numbers by restricted pfaffian functions and the exponential function.
\newblock \emph{Journal of the American Mathematical Society}, 9\penalty0
  (4):\penalty0 1051--1094, 1996.

\bibitem[Zou et~al.(2018)Zou, Cao, Zhou, and Gu]{zou_deep_opt}
Difan Zou, Yuan Cao, Dongruo Zhou, and Quanquan Gu.
\newblock Stochastic gradient descent optimizes over-parameterized deep relu
  networks.
\newblock \emph{arXiv preprint arXiv:1811.08888}, 2018.

\end{thebibliography}
\bibliographystyle{plainnat}

\appendix

\section{Experimental setup}
\label{app_sec:empirical}

The goal of the experiments is to illustrate that directional convergence is
a clear, reliable phenomenon.  Below we detail the setup for the two types of experiments:
contour plots in \Cref{fig:contours}, and margin plots in \Cref{fig:margins} (with
ResNet here in \Cref{fig:margins:resnet}).

\paragraph{Data.}
\Cref{fig:contours} used two-dimensional synthetic data in order to capture the entire
prediction surface; data was generated by labeling points in the plane
with a random network (which included a bias term),
and then deleting low-margin points.
Then, when training from scratch to produce the contours,
data was embedded in $\R^3$ by appending a $1$;
this added bias made the maximum margin network much simpler.

\Cref{fig:margins} used the standard \cifar dataset in its 10 class
configuration \citep{cifar}.  There are 50,000 data points, each with 3072
dimensions, organized into $32\times32$ images with 3 color channels.

\paragraph{Models.}
A few simple models both inside and outside our technical assumptions were used.
All code was implemented in PyTorch \citep{pytorch_official_citation}.

\Cref{fig:contours} worked with a style of 2-layer network which appears widely
throughout theoretical investigations:
specifically, there is first a wide linear layer
(in our case, $10,000$ nodes), then a \emph{squared} ReLU layer,
and then a layer of random signs which is not trained.
This squared ReLU network with one trainable layer is 2-homogeneous,
and was chosen both to fit with the alignment guarantee in \Cref{fact:alignment},
and also to amplify differences with the NTK.
Note that this simple architecture is still a universal approximator with non-convex
training.  \Cref{fig:contours:a,fig:contours:c} trained this network,
which can be written as $x\mapsto \sum_j s_j \max\cbr[1]{0,\ip{w_j}{x}}^2$,
where $s_j\in\pm 1$ are fixed random signs and $(w_j)_{j=1}^m$ are the trainable parameters.
\Cref{fig:contours:ntk} trained the corresponding
NTK~\citep{jacot_ntk,du_deep_opt,allen_deep_opt,zou_deep_opt},
meaning the linear predictor
obtained by freezing the network activations, which thus has the form
$x\mapsto \sum_j s_j \ip{v_j}{x} \max\{0, \ip{w_j}{x}\}$,
where $(w_j)_{j=1}^m$ from before are now fixed, and only $(v_j)_{j=1}^m$ are trained.

\Cref{fig:margins} used convolutional networks.
Firstly, \Cref{fig:margins:halexnet}
used ``H-AlexNet'', which is based on
a simplified version of the standard AlexNet \citep{imagenet_sutskever}
as presented in the PyTorch \cifar tutorial
\citep{pytorch_official_citation},
but
with biases disabled in order to give a homogeneous network.  The network ultimately
consists of
ReLU layers, max-pooling layers, linear layers, and convolutional layers,
and is 5-homogeneous.  In particular, H-AlexNet satisfies all conditions we need for
directional convergence.

The two models outside the assumptions were DenseNet (cf. \Cref{fig:margins:densenet}
and ResNet (cf. \Cref{fig:margins:resnet}), used
unmodified from the PyTorch source, namely by invoking
\texttt{torchvision.models.densetnet121}
and \texttt{torchvision.models.resnet18} with argument
\texttt{num\_classes=10}.

\paragraph{Training.}
Training was a basic gradient descent (GD) for \Cref{fig:contours},
and a basic stochastic gradient descent (SGD) for \Cref{fig:margins,fig:margins:resnet}
with a mini-batch size of 512;
there was no weight decay or other regularization, no momentum, etc.;
it is of course an interesting question how more sophisticated optimization schemes,
including AdaGrad and AdaDelta and others,
affect directional convergence and alignment.
Experiments were run to accuracy $10^{-8}$ or greater in order to train significantly
past the point $\cL(W_0) < \ell(0)$ from \Cref{cond:init},
and to better depict directional convergence.

To help reach such small risk, the main ideas were to rewrite the objective functions
to be numerically stable, and secondly to scale the step size by $\nicefrac 1 {\cL(W_{t-1})}$,
which incidentally is consistent with gradient flow on $\alpha$ with exponential loss,
and is moreover an idea found across the margin literature, most
notably as the step size used in AdaBoost \citep{freund_schapire_adaboost}.
This can lead to some numerical instability,
so the step size was reduced if the norm of the induced update was too large,
meaning the norm of the gradient times the step size was too large.
A much more elaborate numerical scheme was reported by
\citet[Appendix L]{kaifeng_jian_margin}, but not used here.

One point worth highlighting is the role of SGD, which seems as though it should have introduced
a great deal of noise into the plots, and after all is outside the assumptions of the paper
(which requires gradient flow, let alone gradient descent).  Though not depicted here,
experiments in \Cref{fig:margins} were also tried on subsampled data and full gradients,
and \Cref{fig:contours} was tried with SGD in place of GD; while gradient descent does result
in smoother plots, the difference is small overall, leaving the rigorous analysis of directional
convergence with SGD as a promising future direction.

\begin{figure}[t!]
  \centering
  \includegraphics[width=0.6\textwidth]{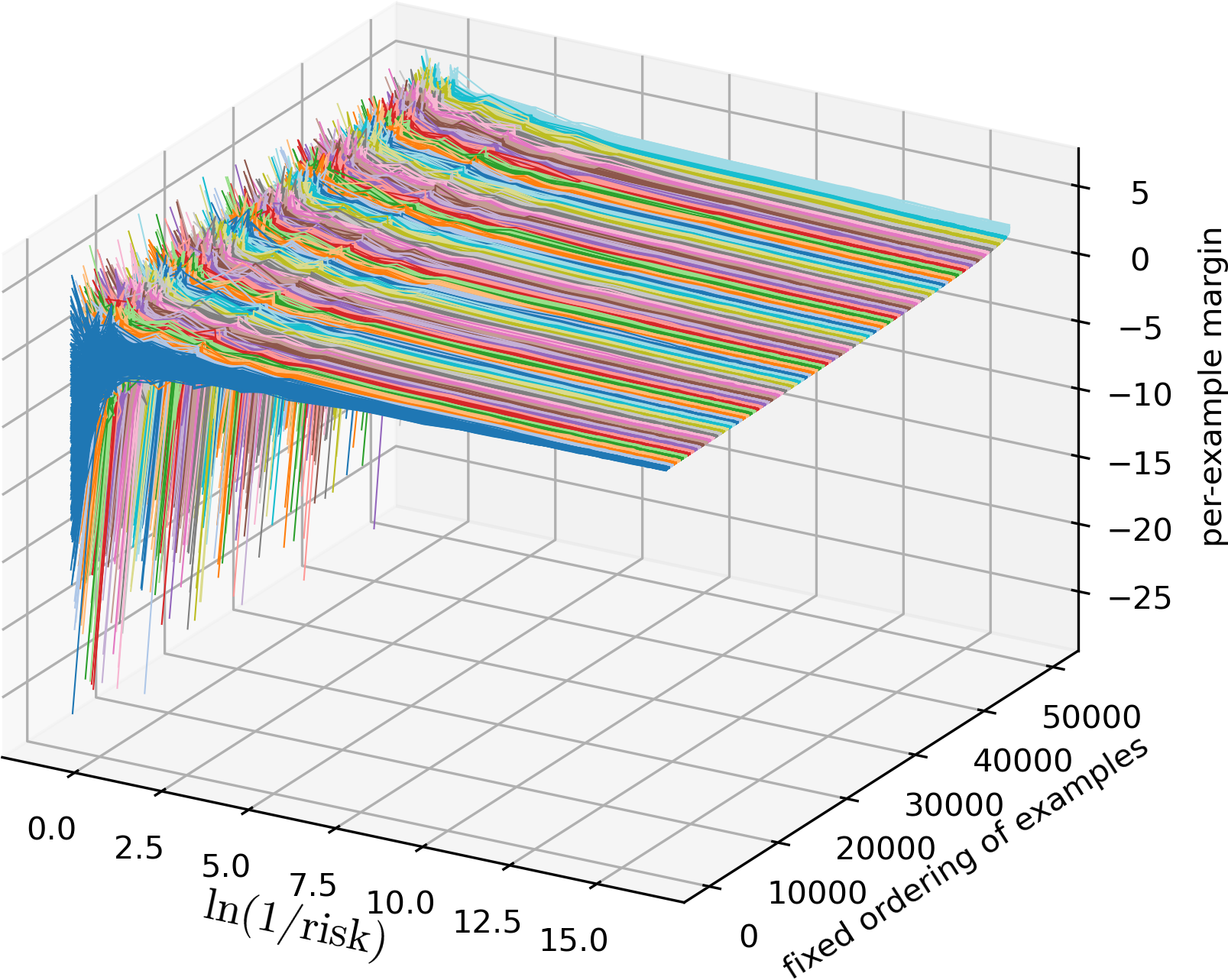}
  \caption{ResNet margins over time, plotted in the same way as \Cref{fig:margins};
  see \Cref{app_sec:empirical} for details.}
  \label{fig:margins:resnet}
\end{figure}

\paragraph{Margin plots.}  A few further words are in order for the margin
plots in \Cref{fig:margins,fig:margins:resnet}.

While margins are well-motivated from generalization and other theoretical perspectives
\citep{spec,margindist_predict,fantastic}, we also use margin plots as a visual surrogate
for prediction surface contour plots from \Cref{fig:contours}, but now for high-dimensional
data, even with high-dimensional outputs.
In particular, \Cref{fig:margins,fig:margins:resnet} track the prediction surface
but restricted to the training set, showing, in a sense, the output trajectory for each
data example.
Since the output dimension is 10 classes,
we convert this to a single real number via the usual multi-class margin
$(x,y) \mapsto \Phi(x;W_t)_y - \max_{j\neq y} \Phi(x;W_t)_j$.

In the case of homogeneous networks, it is natural to normalize this quantity
by $\|W_t\|^L$; for the inhomogeneous cases DenseNet and ResNet, no such normalization
is available. Therefore, for consistency, at each time $t$, margins were normalized
by the median nonnegative margin across all data.

To show the evolution of the margins most clearly, we sorted margins according to the
final margin level, and used this fixed data ordering for all time; as a result, lines
in the plot indeed correspond to trajectories of single examples.
Moreover, we indexed time by the log of the inverse risk,
namely $\ln \nicefrac n {\cL(W_t)}$ in our notation.  While this may seem odd at first,
importantly it washes out the effect of small step-sizes and other implementation choices;
and crucially disallows an artificial depiction of directional convergence by
choosing rapidly-vanishing step sizes.

\section{Results on o-minimal structures}\label{app_sec:o}

An o-minimal structure is a collection $\cS=\{\cS_n\}_{n=1}^\infty$, where
each $\cS_n$ is a set of subsets of $\R^n$ satisfying the following conditions:
\begin{enumerate}
    \item $\cS_1$ is the collection of all finite unions of open intervals and points.

    \item $\cS_n$ includes the zero sets of all polynomials on $\R^n$:
      if $p$ is a polynomial on $\R^n$,
      then $\setdef{x\in\R^n}{p(x)=0}\in\cS_n$.

    \item $\cS_n$ is closed under finite union, finite intersection, and complement.

    \item $\cS$ is closed under Cartesian products:
      if $A\in\cS_m$ and $B\in\cS_n$, then $A\times B\in\cS_{m+n}$.

    \item $\cS$ is closed under projection $\Pi_n$ onto the first $n$ coordinates:
      if $A\in\cS_{n+1}$,
      then $\Pi_n(A)\in\cS_n$.
\end{enumerate}
Given an o-minimal structure $\cS$, a set $A\subset\R^n$ is definable if
$A\in\cS_n$, and a function $f:D\to\R^m$ with $D\subset\R^n$ is definable if the
graph of $f$ is in $\cS_{n+m}$.
Due to the stability under projection, the domain of a definable function is
definable.
In the following we consider an arbitrary fixed o-minimal structure.

\subsection{Basic properties}

A convenient way to construct definable sets and functions is to use
\emph{first-order formulas}:
\begin{itemize}
    \item If $A$ is a definable set, then ``$x\in A$'' is a first-order formula.

    \item If $\phi$ and $\psi$ are first-order formulas, then $\phi\wedge\psi$,
    $\phi\vee\psi$, $\neg\phi$ and $\phi\Rightarrow\psi$ are first-order
    formulas.

    \item If $\phi(x,y)$ is a first-order formula where $x\in\R^n$ and
    $y\in\R^m$, and $A\subset\R^n$ is definable, then $\forall x\in A\phi(x,y)$
    and $\exists x\in A\phi(x,y)$ are first-order formulas.
\end{itemize}
Given a first-order formula, the set of free variables which satisfy the formula
is definable \citep[Appendix A]{van_geo}.
The following basic properties of definable sets and functions can then be shown
(see \citep{van_geo,coste_o,le_lec}).
\begin{enumerate}
    \item Given any $\alpha,\beta\in\R$ and any definable functions
    $f,g:D\to\R$, we have $\alpha f+\beta g$ and $fg$ are definable.
    If $g\ne0$ on $D$, then $f/g$ is definable.
    If $f\ge0$ on $D$, then $f^{1/\ell}$ is definable for any positive integer
    $\ell$.

    \item Given a function $f:D\to\R^m$, let $f_i$ denote the $i$-th coordinate
    of its output.
    Then $f$ is definable if and only if all $f_i$ are definable.

    \item Any composition of definable functions is definable.

    \item Any coordinate permutation of a definable set is definable.
    Consequently, if the inverse of a definable function exists, it is also
    definable.

    \item The image and pre-image of a definable set by a definable function is
    definable.
    Particularly, given any real-valued definable function $f$, all of
    $f^{-1}(0)$, $f^{-1}\del{(-\infty,0)}$ and $f^{-1}\del{(0,\infty)}$
    are definable.

    \item Any combination of finitely many definable functions with disjoint
    domains is definable.
    For example, the pointwise maximum and minimum of definable functions are
    definable.
\end{enumerate}
The proofs are standard and omitted.
To illustrate the idea, we give a proof of the following standard result on the
infimum and supremum operation.
\begin{lemma}\label{fact:inf}
    Let $A\subset\R^{n+1}$ be definable and $\Pi_n$ denote the projection
    onto the first $n$ coordinates.
    Suppose $\inf\setdef{y}{(x,y)\in A}>-\infty$ for all $x\in\Pi_n(A)$, then
    the function from $\Pi_n(A)$ to $\R$ given by
    \begin{align*}
        x\mapsto\inf\setdef{y}{(x,y)\in A}
    \end{align*}
    is definable.
    Consequently, we have:
    \begin{enumerate}
        \item Let $f:D\to\R$ be definable and bounded below, and $g:D\to\R^m$ be
        definable.
        Then $h:g(D)\to\R$ given by $h(y):=\inf_{x\in g^{-1}(y)}f(x)$ is
        definable.

        \item Let $f:D_f\to\R$ and $g:D_g\to\R$ be definable and bounded below,
        then their infimal convolution $h:D_f+D_g\to\R$ given by
        \begin{align*}
            h(z):=\inf\setdef{f(x)+g(y)}{x\in D_f,y\in D_g,x+y=z}
        \end{align*}
        is definable.

        \item A function $f:D\to\R$ is definable if and only if its epigraph is
        definable.

        \item Given a definable set $A$, the function
        $d_A(x):=\inf_{y\in A}\|x-y\|$ is definable, which implies the closure,
        interior and boundary of $A$ are definable.

        \item The lower-semicontinuous envelope of a definable function is
        definable.
    \end{enumerate}
\end{lemma}
\begin{proof}
    Note that the set
    \begin{align*}
        A_\ell:=\setdef{(x,y)}{x\in\Pi_n(A),\textrm{ and }\forall(x,y')\in A,y\le y'}
    \end{align*}
    is definable, since it is given by the following first-order formula:
    \begin{align*}
        (x,y):\quad x\in\Pi_n(A)\ \wedge\ \forall(x',y')\in A\del{(x=x')\Rightarrow(y\le y')}.
    \end{align*}
    Similarly, the set
    \begin{align*}
        A_{\ell u}:=\setdef{(x,y)}{x\in\Pi_n(A),\textrm{ and }\forall(x,y')\in A_\ell,y\ge y'}
    \end{align*}
    is definable, and thus so is $A_\ell\cup A_{\ell u}$, which is the graph of
    the desired function.

    Now we prove the remaining claims.
    \begin{enumerate}
        \item Let $G_f$ denote the graph of $f$, and $G_g$ denote the graph of
        $g$.
        We can just apply the main claim to the following definable set:
        \begin{align*}
            (y,z):\quad y\in g(D)\ \wedge\ \exists(x,y')\in G_g\exists(x',z')\in G_f\del{(x=x')\wedge(y=y')\wedge(z=z')}.
        \end{align*}

        \item First, the Minkowski sum of two definable sets $A$ and $B$ is
        definable:
        \begin{align*}
            z:\quad\exists x\in A\exists y\in B(x+y=z).
        \end{align*}
        Then we can just apply the main claim to the Minkowski sum of the graphs
        of $f$ and $g$.

        \item Let $G_f$ denote the graph of $f$.
        If $G_f$ is definable, then the epigraph is definable:
        \begin{align*}
            (x,y):\quad x\in D\ \wedge\ \forall(x',y')\in G_f\del{(x=x')\Rightarrow(y\ge y')}.
        \end{align*}
        If the epigraph is definable, then $G_f$ is definable due to the main
        claim.

        \item We can just apply the main claim to the set
        \begin{align*}
            (x,r):\quad \exists y\in A\del{\|x-y\|=r}.
        \end{align*}
        The closure of $A$ is just $d_A^{-1}(0)$.
        The interior of $A$ is the complement of $d_{A^c}^{-1}(0)$.
        The boundary is the difference between the closure and interior.

        \item The epigraph of the lower-semicontinuous envelope of $f$ is the
        closure of the epigraph of $f$.
    \end{enumerate}
\end{proof}

As another example, note that the types of networks under discussion are definable.

\begin{lemma}
  \label{fact:ominimal:deepnet}
  Suppose there exist $k,d_0,d_1,\ldots,d_L>0$ and $L$ definable functions
  $(g_1,\ldots,g_L)$ where
  $g_j:\R^{d_0}\times\cdots\times\R^{d_{j-1}}\times\R^k\to\R^{d_j}$.
  Let $h_1(x,W):=g_1(x,W)$, and for $2\le j\le L$,
  \begin{align*}
      h_j(x,W):=g_j\del{x,h_1(x,W),\ldots,h_{j-1}(x,W),W},
  \end{align*}
  then all $h_j$ are definable.
  It suffices if each output coordinate of $g_j$ is the minimum or maximum over
  some finite set of polynomials, which allows for linear, convolutional, ReLU,
  max-pooling layers and skip connections.
\end{lemma}
\begin{proof}
    The definability of $h_j$ can be proved by induction using the fact that
    definability is preserved under composition.
    Next, note that the minimum and maximum of a finite set of polynomials is
    definable.
    Lastly, note that each output coordinate of linear and convolutional layers
    can be written as a polynomial of their input and the parameters; each
    output coordinate of a ReLU layer is the maximum of two polynomials; each
    output of a max-pooling layer is a maximum of polynomials.
    Skip connections are allowed by the definition of $h_j$.
\end{proof}

Below are some useful properties of definable functions.
\begin{proposition}[name={\citep[Exercise 2.7]{le_lec}}]\label{fact:1d_limit}
    Given a definable function $f:(a,b)\to\R$ where $-\infty\le a<b\le\infty$,
    it holds that $\lim_{x\to a^+}f(x)$ and $\lim_{x\to b^-}f(x)$ exist in
    $\R\cup\{-\infty,+\infty\}$.
\end{proposition}
\begin{proof}
    We consider $\lim_{x\to a^+}f(x)$ where $a\in\R$; the other cases can be
    handled similarly.
    If $\lim_{x\to a^+}f(x)$ does not exist, then there exists $k\in\R$ such
    that $\limsup_{x\to a^+}f(x)>k>\liminf_{x\to a^+}f(x)$.
    In other words, for any $\epsilon>0$, there exists
    $x_1,x_2\in(a,a+\epsilon)$ such that $f(x_1)>k$ and $f(x_2)<k$.
    However, since $g:=f-k$ is definable on $(a,b)$, it holds that
    $g^{-1}\del{(-\infty,0)}$, and $g^{-1}(0)$, and $g^{-1}\del{(0,\infty)}$ are
    all definable, and thus they are all finite unions of open intervals and
    points.
    It then follows that there exists $\epsilon_0>0$ such that $g=f-k$ has a
    constant sign (i.e., $>0$, $=0$ or $<0$) on $(a,a+\epsilon_0)$, which is a
    contradiction.
\end{proof}

\begin{theorem}[name={Monotonicity Theorem~\citep[Theorem 4.1]{van_geo}}]
    \label{fact:monotonicity}
    Given a definable function $f:(a,b)\to\R$ where $-\infty\le a<b\le\infty$,
    there exist $a_0,\ldots,a_k,a_{k+1}$ with $a=a_0<a_1<\ldots<a_k<a_{k+1}=b$
    such that for all $0\le i\le k$, it holds on $(a_i,a_{i+1})$ that $f$ is
    $C^1$ and $f'$ has a constant sign (i.e., $>0$, $=0$ or $<0$).
\end{theorem}

\Cref{fact:1d_limit} and \Cref{fact:monotonicity} imply the following result
which we need later.
\begin{lemma}\label{fact:bounded_curve}
    Given a $C^1$ definable curve $\gamma:[0,\infty)\to\R^n$ such that
    $\lim_{s\to\infty}\gamma(s)$ exists and is finite, it holds that the path
    swept by $\gamma$ has finite length.
\end{lemma}
\begin{proof}
    Let $z:=\lim_{s\to\infty}\gamma(s)$.
    Since $\enVert{z-\gamma(s)}$ is definable, either it is $0$ for all large
    enough $s$, or it is positive for all large enough $s$.
    In the first case, since $\gamma$ is $C^1$, it has finite length.
    In the second case, \Cref{fact:monotonicity} implies that there exists an
    interval $[a,\infty)$ on which $\enVert{z-\gamma(s)}>0$ and
    $\dif\,\enVert{z-\gamma(s)}/\dif s<0$, and thus $\enVert{\gamma'(s)}>0$.
    Let
    \begin{align*}
        \lim_{s\to\infty}\frac{z-\gamma(s)}{\enVert{z-\gamma(s)}}=u,\quad\textrm{and}\quad\lim_{s\to\infty}\frac{\gamma'(s)}{\enVert{\gamma'(s)}}=v.
    \end{align*}
    The existence of the above limits is guaranteed by \Cref{fact:1d_limit}.
    Note that $\langle u,v\rangle$ is equal to
    \begin{align*}
        \lim_{s\to\infty}\ip{\frac{z-\gamma(s)}{\enVert{z-\gamma(s)}}}{v}=\lim_{s\to\infty}\frac{\int_s^\infty\ip{\gamma'(\tau)}{v}\dif\tau}{\enVert{z-\gamma(s)}}=\lim_{s\to\infty}\frac{\int_s^\infty\enVert{\gamma'(\tau)}\ip{\gamma'(\tau)/\enVert{\gamma'(\tau)}}{v}\dif\tau}{\enVert{z-\gamma(s)}}.
    \end{align*}
    Since $\gamma'(s)/\enVert{\gamma'(s)}\to v$, given any $\epsilon>0$, for
    large enough $s$ it holds that
    $\ip{\gamma'(s)/\enVert{\gamma'(s)}}{v}\ge1-\epsilon$, and thus
    \begin{align*}
        \langle u,v\rangle=\lim_{s\to\infty}\frac{\int_s^\infty\enVert{\gamma'(\tau)}\ip{\gamma'(\tau)/\enVert{\gamma'(\tau)}}{v}\dif\tau}{\enVert{z-\gamma(s)}}\ge(1-\epsilon)\lim_{s\to\infty}\frac{\int_s^\infty\enVert{\gamma'(\tau)}\dif\tau}{\enVert{z-\gamma(s)}}\ge1-\epsilon,
    \end{align*}
    which implies that $u=v$.
    Since $\epsilon>0$ was arbitrary, then
    \begin{align*}
        \lim_{s\to\infty}\frac{\int_s^\infty\enVert{\gamma'(\tau)}\dif\tau}{\enVert{z-\gamma(s)}}=1,
    \end{align*}
    which implies that $\gamma$ has finite length.
\end{proof}

The following Curve Selection Lemma is crucial in proving the
Kurdyka-\L{}ojasiewicz inequalities.

\begin{lemma}[name={Curve Selection~\citep[Proposition 1]{kurdyka_grad}}]
    \label{fact:curve_sel}
    Given a definable set $A\in\R^n$ and $x\in\overline{A\setminus\{x\}}$, there
    exists a definable curve $\gamma:[0,1]\to\R^n$ which is $C^1$ on $[0,1]$ and
    satisfies $\gamma(0)=x$ and $\gamma\del{(0,1]}\subset A\setminus\{x\}$.
\end{lemma}

We also need the following version at infinity, from
\citep[Lemma 2]{nemethi_milnor} and \citep[Lemma 3.4]{kurdyka_sard}.
\begin{lemma}[Curve Selection at Infinity]\label{fact:curve_sel_infty}
    Given a definable set $A\in\R^n$, a definable function $f:A\to\R$,
    and a sequence $x_i$ in $A$ such that $\lim_{i\to\infty}\|x_i\|=\infty$
    and $\lim_{i\to\infty}f(x_i)= y$, there exists a positive constant $a$ and
    a $C^1$ definable curve $\rho:[a,\infty)\to A$ such that
    $\enVert{\rho(s)}=s$, and $\lim_{s\to\infty}f\del{\rho(s)}= y$.
\end{lemma}
\begin{proof}
    For any $x\in\R^n$, let $x(j)$ denote the $j$-th coordinate of $x$,
    and consider the definable map $\psi : A \to \R^{n+2}$
    given by
    \begin{align*}
      \psi(x):=\del{\frac{x(1)}{\sqrt{1+\|x\|^2}},\ldots,\frac{x(n)}{\sqrt{1+\|x\|^2}},\frac{1}{\sqrt{1+\|x\|^2}}, f(x)}.
    \end{align*}
    By construction, the first $n+1$ coordinates of $\psi(x)$ are bounded for all $x$;
    since furthermore $\lim_{i\to\infty} f(x_i) = y$ with $\lim_{i\to\infty} \|x_i\|\to\infty$,
    then $\psi$ has an accumulation point $(u,0,y)$ for some $\|u\|=1$,
    where $(u,0,y) \in \overline{\psi(A)\setminus\{(u,0,y)\}}$.
    We can therefore apply \Cref{fact:curve_sel},
    obtaining a $C^1$ definable curve
    $\gamma:[0,1]\to\R^{n+2}$ such that $\gamma(0)=(u,0,y)$ and
    $\gamma\del{(0,1]}\subset\psi(A)$.

    With this in hand, define a curve $\rho_0:[1,\infty) \to A$ as
    \begin{align*}
        \rho_0(s):=\psi^{-1}\del{\gamma\del{\frac{1}{s}}},
    \end{align*}
    which is $C^1$ definable and satisfies
    $\lim_{s\to\infty}\enVert{\rho_0(s)}=\infty$ and
    $\lim_{s\to\infty}f\del{\rho_0(s)}= y$.
    \Cref{fact:monotonicity} implies that $\dif\,\enVert{\rho_0(s)}/\dif s$ is
    positive and continuous for all large enough $s$; to finish the proof,
    we may obtain a $C^1$ definable $\rho$ from $\rho_0$ via reparameterization
    (i.e., composing $\rho_0$ with some other $C^1$ definable function from $\R$ to $\R$)
    so that $\enVert{\rho(s)}=s$ on $[a,\infty)$ for some $a\in\R$.
\end{proof}

\subsection{Clarke subdifferentials}\label{app_sec:o_clarke}

Here we prove the definability of Clarke subdifferential, and a \emph{chain
rule along arcs} which is crucial in our analysis.

Here is a standard result on the definability of (Fr\'{e}chet) derivatives:
given a definable function $f:D\to\R$ with an open domain $D$, the set
\begin{align*}
    \setdef{(x,x^*)}{f\textrm{ is Fr\'{e}chet differentiable at }x,\nabla f(x)=x^*}
\end{align*}
is definable, since it is given by the following first-order formula:
\begin{align*}
    (x,x^*):\quad & x\in D\ \wedge \\
     & \forall\epsilon>0\exists\delta>0\forall x'\in D\del{(\|x-x'\|<\delta)\Rightarrow f(x')-f(x)-\langle x^*,x'-x\rangle<\epsilon\|x-x'\|}.
\end{align*}

Now consider a locally Lipschitz definable function $f:D\to\R$ with an open
domain $D$.
Local Lipschitz continuity ensures that G\^{a}teaux and Fr\'{e}chet
differentiability coincide \citep[Exercise 6.2.5]{borwein_lewis}, and $f$ is
differentiable a.e. \citep[Theorem 9.1.2]{borwein_lewis}.
Recall that the Clarke subdifferential at $x\in D$ is defined as
\begin{align*}
    \partial f(x):=\conv\setdef{\lim_{i\to\infty}\nabla f(x_i)}{x_i\in D,\nabla f(x_i)\textrm{ exists},\lim_{i\to\infty}x_i=x},
\end{align*}
and that $\cs f(x)$ denotes the unique minimum-norm subgradient.
Similarly to the gradients, the following result holds for the Clarke
subdifferentials.
\begin{lemma}\label{fact:clarke_def}
    Given a locally Lipschitz definable function $f:D\to\R$ with an open domain
    $D\subset\R^n$, the set
    \begin{align*}
        \Gamma:=\setdef{(x,x^*)}{x\in D,x^*\in\partial f(x)}
    \end{align*}
    is definable.
    Moreover, the function $D\ni x\mapsto \cs f(x)$ is definable.
\end{lemma}
\begin{proof}
    Let $D':=\setdef{x\in D}{\nabla f(x)\textrm{ exists}}$, which is definable.
    The set $A$ given by
    \begin{align*}
        (x,y):\quad x\in D\ \wedge\ \forall\epsilon>0\exists x'\in D'\del{\|x-x'\|<\epsilon}\wedge\del{\enVert{y-\nabla f(x')}<\epsilon}
    \end{align*}
    is also definable.
    Now by Carath\'{e}odory's Theorem, $\Gamma$ is given by
    \begin{align*}
        (x,x^*):\quad & \exists(x_1,x_1^*),\ldots,(x_{n+1},x_{n+1}^*)\in A\exists\lambda_1,\ldots,\lambda_{n+1}\ge0 \\
         & (x_1=x)\wedge\cdots\wedge(x_{n+1}=x)\wedge\del{\sum_{i=1}^{n+1}\lambda_i=1}\wedge\del{\sum_{i=1}^{n+1}\lambda_ix_i^*=x^*}.
    \end{align*}
    It then follows from \Cref{fact:inf} that $x\mapsto\enVert{\cs f(x)}$ and
    $x\mapsto\cs f(x)$ are definable.
\end{proof}

The following \emph{chain rule} is important in our analysis; it allows us to
use $\cs f$ in many places that seem to call on $\nabla f$.
It is basically from \citep[Theorem 5.8 and Lemma 5.2]{davis_tame}, though we
detail how their proof handles our slight extension.
\begin{lemma}\label{fact:chain}
    Given a locally Lipschitz definable $f:D\to\R$ with an open domain $D$, for
    any interval $I$ and any arc $z:I\to D$, it holds for a.e. $t\in I$ that
    \begin{align*}
        \frac{\dif f(z_t)}{\dif t}=\ip{z_t^*}{\frac{\dif z_t}{\dif t}},\quad\textrm{for all }z_t^*\in\partial f(z_t).
    \end{align*}
    Moreover, for the gradient flow in \cref{eq:gf}, it holds for a.e. $t\ge0$
    that $\dif W_t/\dif t=-\cs\cL(W_t)$ and
    $\dif \cL(W_t)/\dif t=-\enVert{\cs\cL(W_t)}^2$.
\end{lemma}
\begin{proof}
  The first part is proved in \citep[Theorem 5.8]{davis_tame} when $D=\R^n$ and
  $I=[0,\infty)$, but actually holds in general as verified below.
  Note that for any $t\in I$ excluding the endpoints, since $f$ is locally
  Lipschitz, there exists a neighborhood $U$ of $z(t)$ on which $f$ is
  $K$-Lipschitz continuous.
  Let $g$ denote the infimal convolution of $f|_U$ and $K\|\cdot\|$.
  It follows that $g$ is definable (\Cref{fact:inf}) and $K$-Lipschitz
  continuous on $\R^n$, and $f=g$ on $U$ \citep[Exercise 7.1.2]{borwein_lewis}.
  Take an interval $[a,b]\ni t$ with rational endpoints such that
  $z\del{[a,b]}\subset U$, and define the absolutely continuous curve
  $\tilde{z}:[0,\infty)\to D$ as $\tilde{z}(t)=z(a+t)$ for $t\in[0,b-a]$, and
  $\tilde{z}(t)=z(b)$ for $t>b-a$.
  Applying \citep[Theorem 5.8]{davis_tame} to $g$ and $\tilde{z}$ gives that the
  chain rule holds for $f$ and $z$ a.e. on $[a,b]$.
  Since this holds for any $t\in I$, and there are only countably many intervals
  with rational endpoints, it follows that the chain rule holds a.e. for $f$ and
  $z$ on $I$.
  The second claim of \Cref{fact:chain} can be proved in the same way as
  \citep[Lemma 5.2]{davis_tame}.
\end{proof}

\subsection{Kurdyka-\L{}ojasiewicz inequalities}\label{app_sec:loja}

\paragraph{Asymptotic Clarke critical values.}

To prove the Kurdyka-\L{}ojasiewicz inequalities, we need the notion of asymptotic
Clarke critical values, introduced in \citep{bolte_clarke}.
Given a locally Lipschitz function $f:D\to\R$ with an open domain $D$, we say
that $a\in\R\cup\{+\infty,-\infty\}$ is an asymptotic Clarke critical value of
$f$ if there exists a sequence $(x_i,x_i^*)$ where $x_i\in D$ and
$x_i^*\in\partial f(x_i)$, such that $\lim_{i\to\infty}(1+\|x_i\|)\|x_i^*\|=0$
and $\lim_{i\to\infty}f(x_i)=a$.

We have the following result regarding the asymptotic Clarke critical values of
a definable function, which is basically from \citep[Corollary 9]{bolte_clarke}.
\begin{lemma}\label{fact:accv}
    Given a locally Lipschitz definable function $f:D\to\R$ with an open domain
    $D$, it holds that $f$ has finitely many asymptotic Clarke critical values.
\end{lemma}
To state the proof in a bit more detail,
\citep[Corollary 9]{bolte_clarke} shows that if $f$ is lower semi-continuous and
$f>-\infty$, then $f$ has finitely many asymptotic Clarke critical values.
To get \Cref{fact:accv}, we just need to apply \citep[Corollary 9]{bolte_clarke}
to the lower semi-continuous envelopes of $f|_{f^{-1}\del{(0,\infty)}}$ and
$-f|_{f^{-1}\del{(-\infty,0)}}$.

\paragraph{The bounded setting.}

Here we consider the case where the domain of $f$ is bounded.
\citep[Theorem 1]{kurdyka_grad} gives a Kurdyka-\L{}ojasiewicz inequality
assuming $f$ is differentiable; below we extend it to the locally Lipschitz
setting.
\begin{lemma}\label{fact:bounded_loja}
    Given a locally Lipschitz definable function $f:D\to\R$ with an open bounded
    domain $D$, there exists $\nu>0$ and a definable desingularizing function
    $\Psi$ on $[0,\nu)$ such that
    \begin{align*}
        \Psi'\del{f(x)}\enVert{\cs f(x)}\ge1
    \end{align*}
    for any $x\in f^{-1}\del{(0,\nu)}$.
\end{lemma}
\begin{proof}
    Since $f$ is definable, $f(D)$ is also definable, and thus is a finite union
    of open intervals and points.
    It follows that either there exists $\epsilon>0$ such that
    $(0,\epsilon)\cap f(D)=\emptyset$, in which case the claim trivially holds;
    otherwise we are free to choose $\epsilon>0$ such that $(0,\epsilon)\subset f(D)$.
    In the second case, define $\phi:(0,\epsilon)\to\R$ as
    \begin{align*}
        \phi(z):=\inf\setdef{\enVert{\cs f(x)}}{f(x)=z}.
    \end{align*}
    By \Cref{fact:inf,fact:clarke_def}, $\phi$ is definable.
    \Cref{fact:accv} implies that there are only finitely many asymptotic Clarke
    critical values on $(0,\epsilon)$, and thus there exists
    $\epsilon'\in(0,\epsilon)$ such that on $(0,\epsilon')$ there is no
    asymptotic Clarke critical value and $\phi(z)>0$.

    Now consider the definable set
    \begin{align*}
        A:=\setdef{x\in f^{-1}\del{(0,\epsilon')}}{\enVert{\cs f(x)}\le2\phi\del{f(x)}}.
    \end{align*}
    It follows that there exists a sequence $x_i$ in $A$ such that $f(x_i)\to0$.
    Since the domain of $f$ is bounded, $x_i$ has an accumulation point $y$.
    Applying \Cref{fact:curve_sel} to the graph of $f|_A$, we have that there
    exists a $C^1$ definable curve $(\rho,h):[0,1]\to\R^{n+1}$ such that
    $\rho(0)=y$, and $h(0)=0$, and $\rho\del{(0,1]}\subset A$, and
    $h(s)=f\del{\rho(s)}$ on $(0,1]$.
    \begin{enumerate}
        \item Since $\rho$ is $C^1$ on $[0,1]$, there exists $B>0$ such that
        $\enVert{\rho'(s)}\le B$ on $[0,1]$.

        \item Since $h$ is definable, $h(0)=0$, and $h(s)>0$ on $(0,1]$,
        \Cref{fact:monotonicity} implies that there exists a constant
        $\omega\in(0,1]$ such that $h'(s)>0$ on $(0,\omega)$.

        \item \Cref{fact:chain} implies that for a.e. $s\in(0,\omega)$,
        \begin{align}\label{eq:chain_tmp}
            h'(s)-\ip{\cs f\del{\rho(s)}}{\rho'(s)}=0.
        \end{align}
        Since the left hand side of \cref{eq:chain_tmp} is definable, it can
        actually be nonzero only for finitely many $s$, and thus is equal to $0$
        on some interval $(0,\mu)$ where $\mu\le\omega$.

        \item Let $\nu=h(\mu)$, the Inverse Function Theorem implies that $\Psi:(0,\nu)\to(0,2B\mu)$ given by $\Psi(z):=2Bh^{-1}(z)$ is also $C^1$
        definable with a positive derivative, and $\lim_{z\to0}\Psi(z)=0$.
    \end{enumerate}
    Now for any $x\in f^{-1}\del{(0,\nu)}$, let $s=h^{-1}\del{f(x)}$, we have
    \begin{align*}
        \Psi'\del{f(x)}\enVert{\cs f(x)} & =\frac{2B}{h'\del{s}}\enVert{\cs f(x)} & \textrm{(Inverse Function Theorem)} \\
         & \ge \frac{2B}{h'\del{s}}\cdot \frac{1}{2}\enVert{\cs f\del{\rho(s)}} & \textrm{(Definition of $A$)}\\
         & =\frac{B\enVert{\cs f\del{\rho(s)}}}{\ip{\cs f\del{\rho(s)}}{\rho'(s)}}\ge1. & \textrm{(Bullet 3 above \& Cauchy-Schwarz)}
    \end{align*}
\end{proof}

\paragraph{The unbounded setting.}

The unbounded setting is more complicated: to show directional convergence, we
need two Kurdyka-\L{}ojasiewicz inequalities (cf. \Cref{fact:unbounded_loja_1,fact:unbounded_loja_2}), depending on the
relationship between the spherical and radial parts of $\cs f$.

Given a locally Lipschitz definable function $f:D\to\R$ with an open domain
$D\subset\setdef{x}{\|x\|>1}$, recall that $\csr f(x)$ and $\csp f(x)$ denote
the radial part and spherical part of $\cs f(x)$ respectively, which are both
definable.
Given $\epsilon,c,\eta>0$, let
\begin{align*}
    U_{\epsilon,c,\eta}:=\setdef{x\in D}{f(x)\in(0,\epsilon),\enVert{\csp f(x)}\ge c\|x\|^\eta\enVert{\csr f(x)}}.
\end{align*}
In any o-minimal structure, $U_{\epsilon,c,\eta}$ is definable if $\eta$ is
rational.
Now we prove \Cref{fact:unbounded_loja_1}, a Kurdyka-\L{}ojasiewicz inequality
on some $U_{\nu,c,\eta}$, using ideas from
\citep[Proposition 6.3]{kurdyka_quasi}.
\begin{proof}[Proof of \Cref{fact:unbounded_loja_1}]
    Similarly to the proof of \Cref{fact:bounded_loja}, we only need to consider
    the case where there exists $\epsilon>0$ such that
    $(0,\epsilon)\subset f(D)$.
    Without loss of generality, we can assume $\eta$ is rational, since
    otherwise we can consider any rational $\eta'\in(0,\eta)$.
    Therefore $U_{\epsilon,c,\eta}$ is definable, and so is
    $f(U_{\epsilon,c,\eta})$.
    If there exists $\epsilon'>0$ such that
    $f(U_{\epsilon,c,\eta})\cap(0,\epsilon')=\emptyset$, then
    \Cref{fact:unbounded_loja_1} trivially holds; therefore we assume that there
    exists $\epsilon'>0$ such that $f(U_{\epsilon',c,\eta})=(0,\epsilon')$.
    By \Cref{fact:accv}, we can also make $\epsilon'$ small enough so that there
    is no asymptotic Clarke critical value on $(0,\epsilon')$.
    Define $\phi:(0,\epsilon')\to\R$ as
    \begin{align*}
        \phi(z):=\inf\setdef{\|x\|\enVert{\cs f(x)}}{x\in U_{\epsilon',c,\eta},f(x)=z}.
    \end{align*}
    Since there is no asymptotic Clarke critical value on $(0,\epsilon')$, it
    holds that $\phi(z)>0$.

    Consider the definable set
    \begin{align*}
        A:=\setdef{x\in U_{\epsilon',c,\eta}}{\|x\|\enVert{\cs f(x)}\le2\phi\del{f(x)}}.
    \end{align*}
    Since $f(U_{\epsilon',c,\eta})=(0,\epsilon')$ as above,
    there exists a sequence $x_i$ in $A$ such that $f(x_i)\to 0$.
    If the $x_i$ are bounded, then the claim follows from the proof of
    \Cref{fact:bounded_loja} and $D\subset\setdef{x}{\|x\|>1}$.
    If the $x_i$ are unbounded, then without loss of generality
    (e.g., by taking a subsequence) we can assume
    $\|x_i\|\to\infty$.
    \Cref{fact:curve_sel_infty} asserts that there exists a $C^1$ definable
    curve $\rho:[a,\infty)\to A$ such that $\enVert{\rho(s)}=s$ and
    $\lim_{s\to\infty}f\del{\rho(s)}=0$.
    Let $h(s):=f\del{\rho(s)}$, and
    $\rho_r'(s):=\ip{\rho'(s)}{\rho(s)}\rho(s)/s^2$ denote the radial part of
    $\rho'(s)$, and $\rho_\perp'(s):=\rho'(s)-\rho_r'(s)$ denote the spherical
    part of $\rho'(s)$.
    \begin{enumerate}
        \item \Cref{fact:monotonicity} implies that $h'$ is negative and
        continuous on some interval $[\omega,\infty)$.

        \item As in the proof of \Cref{fact:bounded_loja}, it follows from
        \Cref{fact:chain} that there exists $\mu\ge\omega$, such that
        \begin{align*}
            h'(s)-\ip{\cs f\del{\rho(s)}}{\rho'(s)}=0
        \end{align*}
        for all $s\in[\mu,\infty)$.

        \item Note that for all $s\in[\mu,\infty)$,
        \begin{align*}
            \envert{h'(s)}=\envert{\ip{\cs f\del{\rho(s)}}{\rho'(s)}} & =\envert{\ip{\csr f\del{\rho(s)}}{\rho_r'(s)}+\ip{\csp f\del{\rho(s)}}{\rho_\perp'(s)}} \\
             & \le\enVert{\csr f\del{\rho(s)}}+\enVert{\csp f\del{\rho(s)}}\enVert{\rho_\perp'(s)} \\
             & \le\del{\frac{1}{cs^\eta}+\enVert{\rho_\perp'(s)}}\enVert{\csp f\del{\rho(s)}}
        \end{align*}
        since $\enVert{\rho_r'(s)}=1$ and $\rho([a,\infty)) \subset U_{\epsilon',c,\eta}$.
        Let $\tilde{\rho}(s):=\rho(s)/s$, we have
        \begin{align*}
            \frac{\dif\tilde{\rho}(s)}{\dif s}=\frac{\rho_\perp'(s)}{s}.
        \end{align*}
        Since $\tilde{\rho}(s)$ is a $C^1$ definable curve on the unit sphere, \Cref{fact:1d_limit} and \Cref{fact:bounded_curve} imply that
        $\enVert{\rho_\perp'(s)}/s$ is integrable on $[\mu,\infty)$.
        Therefore
        \begin{align*}
            \envert{h'(s)}\le-g'(s)\cdot s\enVert{\csp f\del{\rho(s)}},
        \end{align*}
        where
        \begin{align*}
            g(s):=\int_s^\infty\del{\frac{1}{c\tau^{1+\eta}}+\frac{\enVert{\rho_\perp'(\tau)}}{\tau}}\dif\tau.
        \end{align*}
    \end{enumerate}
    Let $\nu=h(\mu)$, and define $\Psi:(0,\nu)\to\R$ as
    \begin{align*}
        \Psi(z):=2g\del{h^{-1}(z)}.
    \end{align*}
    It holds that $\lim_{z\to0}\Psi(z)=0$.
    Moreover, for any $x\in U_{\nu,c,\eta}$, let $s=h^{-1}\del{f(x)}$, we have
    \begin{align*}
        \Psi'\del{f(x)}\|x\|\enVert{\cs f(x)}=\frac{2g'(s)}{h'(s)}\|x\|\enVert{\cs f(x)}\ge \frac{2g'(s)}{h'(s)}\cdot \frac{1}{2}s\enVert{\cs f\del{\rho(s)}}\ge1.
    \end{align*}
\end{proof}

Below we prove \Cref{fact:unbounded_loja_2}, a Kurdyka-\L{}ojasiewicz inequality
which is useful outside of $U_{\nu,c,\eta}$.
\begin{proof}[Proof of \Cref{fact:unbounded_loja_2}]
    We first assume that $\lambda$ is rational, and later finish by handling the real case with
    a quick reduction.
    Consider the definable mapping
    $\xi_\lambda:\R^n\setminus\{0\}\to\R^n\setminus\{0\}$ given by
    \begin{align*}
        \xi_\lambda(x):=\frac{x}{\|x\|^{1+\lambda}}.
    \end{align*}
    Note that $\xi_\lambda^{-1}=\xi_{1/\lambda}$.
    If $y=\xi_\lambda(x)$, then $x=\xi_{1/\lambda}(y)$, which has the Jacobian
    \begin{align}
        \frac{\partial(x_1,\ldots,x_n)}{\partial(y_1,\ldots,y_n)}
        &= \frac{\partial\xi_{1/\lambda}(y)}{\partial(y_1,\ldots,y_n)}
        \notag\\
        &= \|y\|^{-(1+\lambda)/\lambda} \del{
          I - \frac{1+\lambda}{\lambda} \frac {y}{\|y\|} \frac {y^\T}{\|y\|}
        }
        \notag\\
        &= \|x\|^{1+\lambda}\del{I-\frac{1+\lambda}{\lambda}\frac{x}{\|x\|}\frac{x^\T}{\|x\|}}.
        \label{eq:jacobian}
    \end{align}
    Define $g:\xi_{\lambda}(D)\to\R$ as
    \begin{align*}
        g(y):=f\del{\xi_\lambda^{-1}(y)}.
    \end{align*}
    Note that $g$ is locally Lipschitz and definable with an open bounded
    domain.
    Therefore \Cref{fact:bounded_loja} implies that there exists $\nu>0$ and a
    definable desingularizing function $\Psi$ on $[0,\nu)$ such that
    \begin{align*}
        \Psi'\del{g(y)}\enVert{\cs g(y)}\ge1
    \end{align*}
    for any $y\in g^{-1}\del{(0,\nu)}$.
    Let $x=\xi_\lambda^{-1}(y)$, it holds that $g$ is differentiable at $y$ if
    and only if $f$ is differentiable at $x$, and by the definition of Clarke
    subdifferential,
    \begin{align*}
        y^*:=\del{\frac{\partial(x_1,\ldots,x_n)}{\partial(y_1,\ldots,y_n)}}^\T\cs f(x)\in\partial g(y).
    \end{align*}
    Therefore \cref{eq:jacobian} implies that
    \begin{align*}
        \enVert{\cs g(y)}\le\|y^*\|=\|x\|^{1+\lambda}\enVert{\csp f(x)-\frac{1}{\lambda}\csr f(x)}\le \max\cbr{1,\frac{1}{\lambda}}\|x\|^{1+\lambda}\enVert{\cs f(x)},
    \end{align*}
    and thus
    \begin{align*}
        \max\cbr{1,\frac{1}{\lambda}}\Psi'\del{f(x)}\|x\|^{1+\lambda}\enVert{\cs f(x)}\ge1,
    \end{align*}
    which finishes the proof for rational $\lambda$.
    To handle real $\lambda >0$, we can apply the above result to
    any rational $\lambda'\in(\lambda/2,\lambda)$.
\end{proof}

\section{Omitted proofs from \Cref{sec:dir}}\label{app_sec:dir}

We first give a generalization of Euler's homogeneous function theorem, which
can also be found in \citep[Theorem B.2]{kaifeng_jian_margin}, but with an
additional requirement of a chain rule.
\begin{lemma}\label{fact:clarke_euler}
    Suppose $f:\R^n\to\R$ is locally Lipschitz and $L$-positively homogeneous
    for some $L>0$, then for any $x\in\R^n$ and any $x^*\in\partial f(x)$,
    \begin{align*}
        \langle x,x^*\rangle=Lf(x).
    \end{align*}
\end{lemma}
\begin{proof}
    Let $D'$ denote the set of $x$ where $f$ is differentiable.
    For any nonzero $x\in D'$, it holds that
    \begin{align*}
        \lim_{\delta\downarrow0}\frac{f(x+\delta x)-f(x)-\ip{\nabla f(x)}{\delta x}}{\delta\|x\|}=0.
    \end{align*}
    Since $f$ is $L$-positively homogeneous, $f(x+\delta x)=(1+\delta)^Lf(x)$,
    and thus
    \begin{align*}
        \lim_{\delta\downarrow0}\frac{\del{(1+\delta)^L-1}f(x)-\ip{\nabla f(x)}{\delta x}}{\delta\|x\|}=0,
    \end{align*}
    which implies $\ip{x}{\nabla f(x)}=Lf(x)$.
    This property trivially holds if $0\in D'$.

    Now consider an arbitrary $x\in\R^n$.
    For any sequence $x_i$ in $D'$ such that $\lim_{i\to\infty}x_i=x$ and
    $\lim_{i\to\infty}\nabla f(x_i)=x^*$, it holds that
    \begin{align*}
        \langle x,x^*\rangle=\lim_{i\to\infty}\ip{x_i}{\nabla f(x_i)}=\lim_{i\to\infty}Lf(x_i)=Lf(x).
    \end{align*}
    Since $\partial f(x)$ consists of convex combinations of such $x^*$,
    \Cref{fact:clarke_euler} holds.
\end{proof}

Next we prove a few technical lemmas.
Recall the definitions of unnormalized and normalized smoothed margin:
given $W\ne0$, let
\begin{align*}
    \alpha(W):=\ell^{-1}\del{\cL(W)},\quad\textrm{and}\quad\ta(W):=\frac{\alpha(W)}{\|W\|^L}.
\end{align*}
Additionally, given any function $f$ which is locally Lipschitz around a nonzero
$W$, let
\begin{align*}
    \csr f(W):=\ip{\cs f(W)}{\tW}\tW\quad\textrm{and}\quad\csp f(W):=\cs f(W)-\csr f(W)
\end{align*}
denote the radial and spherical parts of $\cs f(W)$ respectively.

We first characterize the Clarke subdifferentials of $\alpha$, the unnormalized
smoothed margin.
\begin{lemma}\label{fact:alpha_cs}
    It holds for any $W\in\R^\pc$ that
    \begin{align*}
        \cs\alpha(W)=\frac{\cs\cL(W)}{\ell'\del{\alpha(W)}},\ \ \textrm{and}\ \  \beta(W):=\frac{\langle W,\cs\alpha(W)\rangle}{L}=\frac{\langle W,W^*\rangle}{L}\textrm{ for any }W^*\in\partial\alpha(W).
    \end{align*}
\end{lemma}
\begin{proof}
    Note that $\cL$ is differentiable at $W$ if and only if $\alpha$ is
    differentiable at $W$, and when both gradients exist, the chain rule and
    inverse function theorem together imply that
    \begin{align*}
        \nabla\alpha(W)=\frac{\nL(W)}{\ell'\del{\ell^{-1}\del{\cL(W)}}}=\frac{\nL(W)}{\ell'\del{\alpha(W)}},
    \end{align*}
    whereby the first claim follows from the definition of Clarke
    subdifferential.
    To prove the second claim, the chain rule for Clarke subdifferentials
    \citep[Theorem 2.3.9]{clarke_opt} implies that
    \begin{align*}
        \partial\alpha(W)\subset\conv\del{\sum_{i=1}^{n}\frac{\ell'\del{p_i(W)}}{\ell'\del{\alpha(W)}}\partial p_i(W)},
    \end{align*}
    and thus \Cref{fact:clarke_euler} ensures for any $W^*\in\partial\alpha(W)$,
    \begin{align*}
        \frac{\langle W,W^*\rangle}{L}=\sum_{i=1}^{n}\frac{\ell'\del{p_i(W)}}{\ell'\del{\alpha(W)}}p_i(W)=\beta(W),
    \end{align*}
    which finishes the proof.
\end{proof}

Next we note that the Clarke subdifferentials of $\alpha$ and $\ta$ are strongly
related.
\begin{lemma}\label{fact:ta_cs}
    For any nonzero $W\in\R^\pc$, we have
    \begin{align*}
        \csr\ta(W)=L\frac{\beta(W)-\alpha(W)}{\|W\|^{L+1}}\tW,\quad\textrm{and}\quad\csp\ta(W)=\frac{\csp\alpha(W)}{\|W\|^L}.
    \end{align*}
\end{lemma}
\begin{proof}
    Note that given $W\ne0$, $\alpha$ is differentiable at $W$ if and only if
    $\ta$ is differentiable at $W$, and when both gradients exist,
    \begin{align*}
        \nabla\ta(W)=\frac{\nabla\alpha(W)}{\|W\|^L}-\frac{\alpha(W)\cdot L\|W\|^{L-1}\tW}{\|W\|^{2L}}=\frac{\nabla\alpha(W)}{\|W\|^L}-L \frac{\alpha(W)\tW}{\|W\|^{L+1}}.
    \end{align*}
    By the definition of Clarke subdifferential, for any nonzero $W$,
    \begin{align}\label{eq:ta_cs_tmp}
        \partial\ta(W)=\setdef{\frac{W^*}{\|W\|^L}-L \frac{\alpha(W)\tW}{\|W\|^{L+1}}}{W^*\in\partial\alpha(W)}.
    \end{align}
    The first claim of \Cref{fact:ta_cs} holds since for any
    $W\in\partial\alpha(W)$, by \Cref{fact:alpha_cs},
    \begin{align*}
        \ip{\frac{W^*}{\|W\|^L}-L \frac{\alpha(W)\tW}{\|W\|^{L+1}}}{\tW}=L \frac{\beta(W)}{\|W\|^{L+1}}-L \frac{\alpha(W)}{\|W\|^{L+1}}.
    \end{align*}
    To prove the second claim, note that since $\partial\alpha(W)$ and
    $\partial\ta(W)$ have fixed radial parts, the norms of the whole
    subgradients are minimized if and only if the norms of their spherical parts
    are minimized.
    Due to \cref{eq:ta_cs_tmp}, the norms of the spherical parts of
    $\partial\alpha(W)$ and $\partial\ta(W)$ are minimized simultaneously, and
    the second claim follows.
\end{proof}

The last technical result we need is that $\alpha$ and $\beta$ are close.
\begin{lemma}\label{fact:margins}
    For $\ell\in\{\lexp,\llog\}$ and any $W$ satisfying $\cL(W)<\ell(0)$, it
    holds that
    \begin{align*}
        0<\alpha(W)\le\beta(W)\le\alpha(W)+2\ln(n)+1.
    \end{align*}
\end{lemma}

To prove \Cref{fact:margins}, we need the following result on $\lexp$ and
$\llog$.
Define $\sigma:\R_+\to\R$ by
\begin{align}\label{eq:sigma}
    \sigma(z):=\ell'\del{\ell^{-1}(z)}\ell^{-1}(z),
\end{align}
and $\pi:\R^n\to\R$ by
\begin{align}\label{eq:pi}
    \pi(v):=\ell^{-1}\del{\sum_{i=1}^{n}\ell(v_i)}.
\end{align}
Note that $\alpha(W)=\pi\del{p(W)}$ where $p(W)=\del{p_1(W),\ldots,p_n(W)}$.
\begin{lemma}\label{fact:exp_log}
    For $\ell\in\{\lexp,\llog\}$, it holds that $\sigma$ is super-additive on
    $\del{0,\ell(0)}$, meaning that $\sigma(z_1+z_2)\ge\sigma(z_1)+\sigma(z_2)$
    for any $z_1,z_2>0$ such that $z_1+z_2<\ell(0)$.
    Moreover $\pi$ is concave.
\end{lemma}
\begin{proof}
    For $\lexp(z)=e^{-z}$, we have $\sigma(z)=z\ln(z)$, while for
    $\llog(z)=\ln(1+e^{-z})$, we have $\sigma(z)=(1-e^{-z})\ln(e^z-1)$.
    In both cases $\lim_{z\to0}\sigma(z)=0$, and $\sigma$ is convex on
    $\del{0,\ell(0)}$, which implies super-additivity.

    Turning to concavity of $\pi$, in the case of $\lexp$, it is a standard fact
    in convex analysis that the function $\pi(v)=-\ln\sum_{i=1}^n\exp(-v_i)$ is
    concave \citep[Exercise 3.3.7]{borwein_lewis}.
    For $\llog$, note that
    \begin{align*}
        \frac{\partial\pi}{\partial v_i}=\frac{\ell'(v_i)}{\ell'\del{\ell^{-1}\del{\sum_{i=1}^{n}\ell(v_i)}}}=\frac{\ell'(v_i)}{\exp\del{-S(v)}-1},
    \end{align*}
    where $S(v):=\sum_{i=1}^{n}\ell(v_i)$, and
    \begin{align*}
        \nabla^2\pi(v)=\frac{1}{\exp\del{-S(v)}-1}\diag\del{\ell''(v_1),\ldots,\ell''(v_n)}+\frac{\exp\del{-S(v)}}{\del{\exp\del{-S(v)}-1}^2}\nabla S(v)\nabla S(v)^\T.
    \end{align*}
    We want to show that $\nabla^2\pi(v)\preceq0$, or equivalently
    \begin{align*}
        \del{\exp\del{S(v)}-1}\diag\del{\ell''(v_1),\ldots,\ell''(v_n)}-\nabla S(v)\nabla S(v)^\T\succeq0.
    \end{align*}
    By definition, we need to show that for any $z\in\R^n$,
    \begin{align*}
        \del{\exp\del{S(v)}-1}\sum_{i=1}^{n}\ell''(v_i)z_i^2\ge\del{\sum_{i=1}^{n}\ell'(v_i)z_i}^2.
    \end{align*}
    Note that for $a,b>0$, we have $e^{a+b}-1>(e^a-1)+(e^b-1)$, which implies
    \begin{align*}
        \exp\del{S(v)}-1>\sum_{i=1}^{n}\del{\exp\del{\ell(v_i)}-1}=\sum_{i=1}^{n}e^{-v_i}.
    \end{align*}
    Also note that $e^{-v_i}\ell''(v_i)=\ell'(v_i)^2$, and thus
    \begin{align*}
        \del{\exp\del{S(v)}-1}\sum_{i=1}^{n}\ell''(v_i)z_i^2\ge \sum_{i=1}^{n}e^{-v_i}\sum_{i=1}^{n}\ell''(v_i)z_i^2\ge\del{\sum_{i=1}^{n}\ell'(v_i)z_i}^2.
    \end{align*}
\end{proof}

Using \Cref{fact:exp_log}, we can prove \Cref{fact:margins}.
\begin{proof}[Proof of \Cref{fact:margins}]
    For simplicity, let $p:=\del{p_1(W),\ldots,p_n(W)}$.
    Recall that $\alpha(W)=\pi(p)$, and from the proof of \Cref{fact:alpha_cs} we
    know that
    \begin{align*}
        \beta(W)=\sum_{i=1}^{n}\frac{\ell'\del{p_i(W)}}{\ell'\del{\ell^{-1}\del{\cL(W)}}}p_i(W)=\ip{\nabla\pi(p)}{p}.
    \end{align*}

    By the super-additivity of the function $\sigma$ defined in \cref{eq:sigma},
    we know that
    \begin{align*}
        \sum_{i=1}^{n}\ell'\del{p_i(W)}p_i(W) & =\sum_{i=1}^{n}\ell'\del{\ell^{-1}\del{\ell(p_i(W))}}\ell^{-1}\del{\ell(p_i(W))} \\
         & \le\ell'\del{\ell^{-1}\del{\cL(W)}}\ell^{-1}\del{\cL(W)} \\
         & =\ell'\del{\ell^{-1}\del{\cL(W)}}\alpha(W),
    \end{align*}
    and since $\ell'<0$, we have $\beta(W)\ge\alpha(W)$.

    On the other claim, for $\lexp$, since $\pi$ is concave,
    \begin{align*}
        \beta(W)=\ip{\nabla\pi(p)}{p}=\ip{\nabla\pi(p)}{p-0}\le\pi(p)-\pi(0)=\alpha(W)+\ln(n).
    \end{align*}
    For $\llog$, note that on the interval $\del{0,\ell(0)}$, the function
    $h(z):=\ell'\del{\ell^{-1}(z)}=e^{-z}-1$ is convex with $\lim_{z\to0}h(z)=0$
    and $h'(z)\in(-1,-1/2)$, and thus
    \begin{align*}
        \enVert{\pi(p)}_1=\sum_{i=1}^{n}\frac{\ell'\del{p_i(W)}}{\ell'\del{\ell^{-1}\del{\cL(W)}}}\le2.
    \end{align*}
    Let $c=-\ln\del{\exp\del{\ln(2)/n}-1}\le\ln(n)-\ln\ln(2)$ and $\vec{1}$
    denote the all-ones vector, we have $\pi\del{c\vec{1}}=0$, and
    \begin{align*}
        \beta(W)=\ip{\nabla\pi(p)}{p} & =\ip{\nabla\pi(p)}{p-c\vec{1}}+\ip{\nabla\pi(p)}{c\vec{1}} \\
         & \le\pi(p)-\pi\del{c\vec{1}}+c\enVert{\pi(p)}_1 \\
         & =\alpha(W)+c\enVert{\pi(p)}_1 \\
         & \le\alpha(W)+2\ln(n)-2\ln\ln(2)\le\alpha(W)+2\ln(n)+1.
    \end{align*}
\end{proof}

Now we can prove \Cref{fact:ta_zeta}.
\begin{proof}[Proof of \Cref{fact:ta_zeta}]
    \Cref{fact:chain} implies that for a.e. $t\ge0$,
    \begin{align*}
        \frac{\dif W_t}{\dif t}=-\cs\cL(W_t).
    \end{align*}

    First note that \Cref{cond:init} implies that $\|W_0\|>0$, and moreover
    \citet[Lemma 5.1]{kaifeng_jian_margin} proved that $\dif\|W_t\|/\dif t>0$
    for a.e. $t\ge0$, and thus $\|W_t\|$ is increasing and
    $\|W_t\|\ge\|W_0\|>0$.

    Now we have for a.e. $t\ge0$,
    \begin{align*}
        \frac{\dif\ta(W_t)}{\dif t}=\ip{\cs\ta(W_t)}{-\cs\cL(W_t)}=\ip{\csr\ta(W_t)}{-\csr\cL(W_t)}+\ip{\csp\ta(W_t)}{-\csp\cL(W_t)}.
    \end{align*}
    By \Cref{fact:margins,fact:alpha_cs,fact:ta_cs}, both
    $\ip{\csr\ta(W_t)}{\tW_t}$ and $\ip{-\csr\cL(W_t)}{\tW_t}$ are nonnegative,
    and thus
    \begin{align*}
        \ip{\csr\ta(W_t)}{-\csr\cL(W_t)}=\enVert{\csr\ta(W_t)}\enVert{\csr\cL(W_t)}.
    \end{align*}
    \Cref{fact:alpha_cs,fact:ta_cs} also imply that $\csp\ta(W_t)$ and
    $-\csp\cL(W_t)$ point to the same direction, and thus
    \begin{align*}
        \ip{\csp\ta(W_t)}{-\csp\cL(W_t)}=\enVert{\csp\ta(W_t)}\enVert{\csp\cL(W_t)}.
    \end{align*}

    Now consider $\tW_t$ and $\zeta_t$.
    Since $W_t$ is an arc, and $\|W_t\|\ge\|W_0\|>0$, it follows that $\tW_t$ is
    also an arc.
    Moreover, for a.e. $t\ge0$,
    \begin{align*}
        \frac{\dif\tW_t}{\dif t}=\frac{1}{\|W_t\|}\frac{\dif W_t}{\dif t}-\frac{1}{\|W_t\|}\tW_t\ip{\frac{\dif W_t}{\dif t}}{\tW_t}=\frac{-\csp\cL(W_t)}{\|W_t\|}.
    \end{align*}
    Since $\tW_t$ is an arc, $\dif\tW_t/\dif t$ and $\enVert{\dif\tW_t/\dif t}$
    are both integrable, and by definition of the curve length,
    \begin{align*}
        \zeta_t=\int_0^t\enVert{\frac{\dif\tW_t}{\dif t}}\dif t,
    \end{align*}
    and for a.e. $t\ge0$ we have
    \begin{align*}
        \frac{\dif\zeta_t}{\dif t}=\enVert{\frac{\dif\tW_t}{\dif t}}=\frac{\enVert{\csp\cL(W_t)}}{\|W_t\|}.
    \end{align*}
\end{proof}

Finally we prove the core \Cref{fact:ta_zeta_ratio}, which directly implies
\Cref{fact:dir}.
\begin{proof}[Proof of \Cref{fact:ta_zeta_ratio}]
    Recall that $\ta_t$ denotes $\ta(W_t)$, and $a=\lim_{t\to\infty}\ta_t$.

    First note that if $\ta_{t_0}=a$ for some finite $t_0$, then
    $\dif\ta_t/\dif t=0$ for a.e. $t\ge0$.
    \Cref{fact:ta_zeta} then implies for a.e. $t\ge0$ that
    $\enVert{\csp\cL(W_t)}=0$ and $\dif\zeta_t/\dif t=0$, and then
    \Cref{fact:ta_zeta_ratio} trivially holds.
    Below we assume $\ta_t<a$ for all finite $t\ge0$, and fix an arbitrary
    $\kappa\in(L/2,L)$.
    We consider two cases.
    \begin{enumerate}
        \item \Cref{fact:unbounded_loja_1} implies that there exists $\nu_1>0$
        and a definable desingularizing function $\Psi_1$ on $[0,\nu_1)$,
        such that if $W$ satisfies $\|W\|>1$, and $\ta(W)>a-\nu_1$, and
        \begin{align}\label{eq:ta_zeta_cond_1}
            \enVert{\csp\ta(W)}\ge \frac{\ta_0}{2\ln(n)+1}\|W\|^{L-\kappa}\enVert{\csr\ta(W)},
        \end{align}
        then
        \begin{align}\label{eq:ta_zeta_loja_1}
            \Psi_1'\del{a-\ta(W)}\|W\|\enVert{\cs\ta(W)}\ge1.
        \end{align}

        Now consider $t$ large enough such that $\|W_t\|>1$, and
        $\ta_t>a-\nu_1$, and $\ta_0\|W_t\|^{L-\kappa}/(2\ln(n)+1)\ge1$, and
        moreover assume \cref{eq:ta_zeta_cond_1} holds for $W_t$.
        We have
        \begin{align*}
            \enVert{\csp\ta(W_t)}\ge\enVert{\csr\ta(W_t)},\quad\textrm{and thus}\quad\enVert{\csp\ta(W_t)}\ge \frac{1}{2}\enVert{\cs\ta(W_t)}.
        \end{align*}
        Therefore \Cref{fact:ta_zeta} implies
        \begin{align}\label{eq:ta_zeta_tmp1}
            \frac{\dif\ta_t}{\dif t} & \ge\enVert{\csp\ta(W_t)}\enVert{\csp\cL(W_t)} \nonumber \\
             & =\|W_t\|\enVert{\csp\ta(W_t)}\frac{\dif\zeta_t}{\dif t} \nonumber \\
             & \ge \frac{1}{2}\|W_t\|\enVert{\cs\ta(W_t)}\frac{\dif\zeta_t}{\dif t}.
        \end{align}
        Consequently, \cref{eq:ta_zeta_tmp1,eq:ta_zeta_loja_1} imply that
        \begin{align*}
            \frac{\dif\ta_t}{\dif t}\ge \frac{1}{2\Psi_1'\del{a-\ta_t}}\frac{\dif\zeta_t}{\dif t}.
        \end{align*}

        \item On the other hand, \Cref{fact:unbounded_loja_2} implies that there
        exists $\nu_2>0$ and a definable desingularizing function $\Psi_2$ on
        $[0,\nu_2)$, such that if $\|W\|>1$, and $\ta(W)>a-\nu_2$, then
        \begin{align}\label{eq:ta_zeta_loja_2}
            \max\cbr{1,\frac{2}{2\kappa-L}}\Psi_2'\del{a-\ta(W)}\|W\|^{2\kappa-L+1}\enVert{\cs\ta(W)}\ge1.
        \end{align}
        Now consider $t$ large enough such that $\|W_t\|>1$, and
        $\ta_t>a-\nu_2$, and $\ta_0\|W_t\|^{L-\kappa}/(2\ln(n)+1)\ge1$, and
        moreover
        \begin{align}\label{eq:ta_zeta_cond_2}
            \enVert{\csp\ta(W_t)}\le \frac{\ta_0}{2\ln(n)+1}\|W_t\|^{L-\kappa}\enVert{\csr\ta(W_t)}.
        \end{align}
        Note that \cref{eq:ta_zeta_cond_2} is the opposite to
        \cref{eq:ta_zeta_cond_1}.
        \Cref{fact:alpha_cs,fact:margins} implies that
        \begin{align}\label{eq:csr_alpha_lb}
            \enVert{\csr\alpha(W_t)}=\frac{L\beta(W_t)}{\|W_t\|}\ge \frac{L\alpha(W_t)}{\|W_t\|}=L\ta_t\|W_t\|^{L-1}\ge L\ta_0\|W_t\|^{L-1},
        \end{align}
        while \Cref{fact:ta_cs} implies that
        \begin{align*}
            \enVert{\csr\ta(W_t)}=L\frac{\beta(W_t)-\alpha(W_t)}{\|W_t\|^{L+1}}\le \frac{L(2\ln(n)+1)}{\|W_t\|^{L+1}},
        \end{align*}
        and thus
        \begin{align}\label{eq:ta_zeta_tmp2}
            \enVert{\csr\alpha(W_t)}\ge \frac{\ta_0}{2\ln(n)+1}\|W_t\|^{2L}\enVert{\csr\ta(W_t)}.
        \end{align}
        On the other hand, $\csp\alpha(W_t)=\|W_t\|^L\csp\ta(W_t)$ by
        \Cref{fact:ta_cs}, which implies the following in light of
        \cref{eq:ta_zeta_tmp2,eq:ta_zeta_cond_2}:
        \begin{align*}
            \enVert{\csr\alpha(W_t)} & \ge \frac{\ta_0}{2\ln(n)+1}\|W_t\|^{2L}\enVert{\csr\ta(W_t)} \\
             & \ge\|W_t\|^{L+\kappa}\enVert{\csp\ta(W_t)} \\
             & =\|W_t\|^\kappa\enVert{\csp\alpha(W_t)}.
        \end{align*}
        By \Cref{fact:alpha_cs}, $\cs\alpha(W_t)$ is parallel to $\cs\cL(W_t)$,
        therefore
        \begin{align}\label{eq:ta_zeta_tmp3}
            \enVert{\csr\cL(W_t)}\ge\|W_t\|^\kappa\enVert{\csp\cL(W_t)}.
        \end{align}
        Moreover, if $\ta_0\|W_t\|^{L-\kappa}/(2\ln(n)+1)\ge1$, then the
        triangle inequality implies
        \begin{align*}
            \enVert{\cs\ta(W_t)}\le\enVert{\csp\ta(W_t)}+\enVert{\csr\ta(W_t)}\le \frac{2\ta_0}{2\ln(n)+1}\|W_t\|^{L-\kappa}\enVert{\csr\ta(W_t)},
        \end{align*}
        or
        \begin{align}\label{eq:ta_zeta_tmp4}
            \enVert{\csr\ta(W_t)}\ge \frac{2\ln(n)+1}{2\ta_0}\|W_t\|^{\kappa-L}\enVert{\cs\ta(W_t)}.
        \end{align}
        Now \Cref{fact:ta_zeta} and
        \cref{eq:ta_zeta_tmp3,eq:ta_zeta_tmp4} imply
        \begin{align*}
            \frac{\dif\ta_t}{\dif t} & \ge\enVert{\csr\ta(W_t)}\enVert{\csr\cL(W_t)} \\
             & \ge \frac{2\ln(n)+1}{2\ta_0}\|W_t\|^{2\kappa-L}\enVert{\cs\ta(W_t)}\enVert{\csp\cL(W_t)} \\
             & = \frac{2\ln(n)+1}{2\ta_0}\|W_t\|^{2\kappa-L+1}\enVert{\cs\ta(W_t)}\frac{\dif\zeta_t}{\dif t}.
        \end{align*}
        Then \cref{eq:ta_zeta_loja_2} further implies
        \begin{align*}
            \frac{\dif\ta_t}{\dif t}\ge \frac{2\ln(n)+1}{2\ta_0\max\{1,2/(2\kappa-L)\}}\frac{1}{\Psi_2'\del{a-\ta_t}}\frac{\dif\zeta_t}{\dif t}.
        \end{align*}
    \end{enumerate}

    Since $\Psi_1'-\Psi_2'$ is definable, it is nonnegative or nonpositive on
    some interval $(0,\nu)$.
    Let $\Psi'=\max\{\Psi_1',\Psi_2'\}$ on $(0,\nu)$.
    Now for a.e. large enough $t$ such that  $\|W_t\|>1$, and $\ta_t>a-\nu$,
    and $\ta_0\|W_t\|^{L-\kappa}/(2\ln(n)+1)\ge1$, it holds that
    \begin{align*}
        \frac{\dif\ta_t}{\dif t}\ge \frac{1}{c\Psi'\del{a-\ta_t}}\frac{\dif\zeta_t}{\dif t}
    \end{align*}
    for some constant $c>0$.
    \Cref{fact:ta_zeta_ratio} then follows.
\end{proof}

\section{Omitted proofs from \Cref{sec:alignment}}\label{app_sec:alignment}

We first give the following technical result.
\begin{lemma}\label{fact:grad_hom}
    Suppose $f:\R^n\to\R$ is $L$-positively homogeneous for some $L>0$ and has
    a locally Lipschitz gradient at all nonzero $x\in\R^n$.
    Then $\nabla f$ is $(L-1)$-positively homogeneous: given any nonzero
    $x$ and $c>0$, it holds that
    \begin{align*}
        \nabla f(cx)=c^{L-1}\nabla f(x).
    \end{align*}
    If $\nabla f$ is differentiable at a nonzero $x$, then for any $c>0$, it
    holds that
    \begin{align*}
        \nabla^2f(cx)=c^{L-2}\nabla^2f(x).
    \end{align*}
    Moreover, there exists $K_\sigma>0$ such that for any $\|x\|=1$, if
    $\nabla^2f(x)$ exists, then $\enVert{\nabla^2f(x)}_\sigma\le K_\sigma$.
\end{lemma}
\begin{proof}
    By definition,
    \begin{align*}
        \lim_{\|y\|\downarrow0}\frac{f(x+y)-f(x)-\ip{\nf(x)}{y}}{\|y\|}=0.
    \end{align*}
    On the other hand, by homogeneity,
    \begin{align*}
        f(cx+z)-f(cx)-\ip{c^{L-1}\nf(x)}{z}=c^L\del{f\del{x+\frac{z}{c}}-f(x)-\ip{\nf(x)}{\frac{z}{c}}}.
    \end{align*}
    Therefore
    \begin{align*}
        \lim_{\|z\|\downarrow0}\frac{f(cx+z)-f(cx)-\ip{c^{L-1}\nf(x)}{z}}{\|z\|}=c^{L-1}\lim_{\|z\|\downarrow0}\frac{f\del{x+\frac{z}{c}}-f(x)-\ip{\nf(x)}{\frac{z}{c}}}{\|z/c\|}=0,
    \end{align*}
    which proves the claim.
    The homogeneity of $\nabla^2f$ when it exists can be proved in the same way.

    To get $K_\sigma$, note that for any $\|x\|=1$, there exists an open
    neighborhood $U_x$ of $x$ on which $\nf$ is $K_x$-Lipschitz continuous, and
    thus the spectral norm of $\nabla^2f$ is bounded by $K_x$ when it exists.
    All the $U_x$ form an open cover of the compact unit sphere, and thus has a
    finite subcover, which implies the claim.
\end{proof}

Below we estimate various quantities using \Cref{fact:grad_hom}.
\begin{lemma}\label{fact:margins_sizes}
    Suppose $\ell\in\{\lexp,\llog\}$, all $p_i$ are $L$-positively homogeneous
    for some $L>0$, and all $\nabla p_i$ are locally Lipschitz.
    For any $W$ such that $\cL(W)<\ell(0)$, it holds that $\beta(W)/\|W\|^L$ and
    $\enVert{\na(W)}/\|W\|^{L-1}$ are bounded.
\end{lemma}
\begin{proof}
    Since $p_i(W)$ is continuous, it is bounded on the unit sphere.
    Because it is $L$-positively homogeneous, $p_i(W)/\|W\|^L$ is bounded on
    $\R^\pc$.
    \Cref{fact:margins} implies that
    $\beta(W)-2\ln(n)-1\le\alpha(W)\le\min_{1\le i\le n}p_i(W)$, and it follows
    that $\beta(W)/\|W\|^L$ is bounded.

    Recall that
    \begin{align*}
        \na(W)=\sum_{i=1}^{n}\frac{\partial\pi}{\partial p_i}\nabla p_i(W),
    \end{align*}
    where $\pi$ is defined in \cref{eq:pi} and all partial derivatives are
    evaluated at $p(W):=(p_1(W),\ldots,p_n(W))$.
    It is shown in the proof of \Cref{fact:margins} that
    $\enVert{\pi(p)}_1\le2$.
    Moreover, \Cref{fact:grad_hom} implies that all
    $\enVert{\nabla p_i(W)}/\|W\|^{L-1}$ are bounded.
    Consequently, $\enVert{\na(W)}/\|W\|^{L-1}$ is bounded.
\end{proof}

Recall the definition of $\cJ$:
\begin{align*}
    \cJ(W):=\frac{\enVert{\na(W)}^2}{\|W\|^{2L-2}}.
\end{align*}
If all $\nabla p_i$ are locally Lipschitz, then $\cJ$ is also locally Lipschitz.
We further have the following result.
\begin{lemma}\label{fact:der_norm}
    Under the same conditions as \Cref{fact:margins_sizes}, for any $W$
    satisfying $\cL(W)<\ell(0)$ and any $W^*\in\partial \cJ(W)$,
    \begin{align*}
        \ip{W^*}{-\nL(W)}\le-K\ell'\del{\alpha(W)}\|W\|^{L-2}\sin(\theta)^2
    \end{align*}
    for some constant $K>0$, where $\theta$ denotes the angle between $W$ and
    $-\nL(W)$.
\end{lemma}
\begin{proof}
    Let $D'$ denote the set of $W$ where all $\nabla p_i$ are differentiable,
    and let $S_0$ denote the set of $W$ where $\cL(W)<\ell(0)$.
    We only need to prove the lemma on $D'\cap S_0$, since for any
    $W\in S_0$ it follows from \citep[Theorem 2.5.1]{clarke_opt} that
    \begin{align*}
        \partial\cJ(W)=\conv\setdef{\lim\nabla\cJ(W_i)}{W_i\to W,W_i\in D'\cap S_0}.
    \end{align*}

    Below we fix an arbitrary $W\in D'\cap S_0$.
    All the partial derivatives below with respect to $p_i$ are evaluated at
    $p(W):=(p_1(W),\ldots,p_n(W))$.
    Recall that
    \begin{align*}
        \na(W)=\sum_{i=1}^{n}\frac{\partial\pi}{\partial p_i}\nabla p_i(W),
    \end{align*}
    where $\pi$ is defined in \cref{eq:pi}.
    Since $\nabla p_i$ are also differentiable at $W$, we have
    \begin{align}\label{eq:n2a}
        \nabla^2\alpha(W)=\sum_{i=1}^{n}\sum_{j=1}^{n}\del{\frac{\partial^2\pi}{\partial p_i\partial p_j}\nabla p_i(W)\nabla p_j(W)^\T}+\sum_{i=1}^{n}\frac{\partial\pi}{\partial p_i}\nabla^2p_i(W).
    \end{align}
    Now for any $W\in D'\cap S_0$, we have (recall that $\tW=W/\|W\|$)
    \begin{align*}
        \nabla\cJ(W) & =\frac{2\nabla^2\alpha(W)\na(W)}{\|W\|^{2L-2}}-\frac{\enVert{\na(W)}^2}{\|W\|^{4L-4}}\cdot(2L-2)\|W\|^{2L-3}\tW \\
        & =\frac{2\nabla^2\alpha(W)\na(W)}{\|W\|^{2L-2}}-\frac{(2L-2)\enVert{\na(W)}^2}{\|W\|^{2L}}W,
    \end{align*}
    and thus
    \begin{align}\label{eq:nj_nl}
         & \ \frac{\|W\|^{2L}}{2}\frac{\ip{\nabla\cJ(W)}{-\nL(W)}}{-\ell'\del{\alpha(W)}} \nonumber \\
        = & \ \frac{\|W\|^{2L}}{2}\ip{\nabla\cJ(W)}{\na(W)} \nonumber \\
        = & \ \|W\|^2\na(W)^\T\nabla^2\alpha(W)\na(W)-(L-1)\enVert{\na(W)}^2\ip{W}{\na(W)}.
    \end{align}
    Comparing \cref{eq:n2a,eq:nj_nl}, first note that
    \begin{align*}
        \sum_{i=1}^{n}\sum_{j=1}^{n}\frac{\partial^2\pi}{\partial p_i\partial p_j}\na(W)^\T\nabla p_i(W)\nabla p_j(W)^\T\na(W)\le0,
    \end{align*}
    since $\pi$ is concave by \Cref{fact:exp_log}, and moreover
    \begin{align*}
        \ip{W}{\na(W)}=\sum_{i=1}^{n}\frac{\partial\pi}{\partial p_i}\ip{W}{\nabla p_i(W)}=L \sum_{i=1}^{n}\frac{\partial\pi}{\partial p_i}p_i(W).
    \end{align*}
    Therefore \cref{eq:nj_nl} is upper bounded by
    \begin{align}\label{eq:nj_nl_lb}
        \|W\|^2\sum_{i=1}^{n}\frac{\partial\pi}{\partial p_i}\na(W)^\T\nabla^2p_i(W)\na(W)-L(L-1)\enVert{\na(W)}^2\sum_{i=1}^{n}\frac{\partial\pi}{\partial p_i}p_i(W).
    \end{align}

    Let $\nra(W)$ and $\npa(W)$ denote the radial and spherical part of
    $\na(W)$, respectively.
    Let $\theta$ denote the angle between $W$ and $\na(W)$.
    \Cref{fact:alpha_cs,fact:margins} imply that
    \begin{align*}
        \ip{W}{\na(W)}=L\beta(W)>0,
    \end{align*}
    and thus $\theta$ is between $0$ and $\pi/2$.
    Now \Cref{fact:grad_hom} and the proof of \Cref{fact:clarke_euler} imply
    that
    \begin{align}\label{eq:nj_nl_lb1}
        \|W\|^2\nra(W)^\T\nabla^2p_i(W)\nra(W) & =\cos(\theta)^2\enVert{\na(W_t)}^2W^\T\nabla^2p_i(W)W \nonumber \\
         & =\cos(\theta)^2\enVert{\na(W_t)}^2\cdot L(L-1)p_i(W) \nonumber \\
         & \le\enVert{\na(W_t)}^2\cdot L(L-1)p_i(W).
    \end{align}
    Moreover,
    \begin{align*}
        2\|W\|^2\npa(W)^\T\nabla^2p_i(W)\nra(W) & =2\|W\|\enVert{\na(W)}\cos(\theta)\ip{\npa(W)}{\nabla^2p_i(W)W} \\
         & =2(L-1)\|W\|\enVert{\na(W)}\cos(\theta)\ip{\npa(W)}{\nabla p_i(W)},
    \end{align*}
    and thus by \Cref{fact:alpha_cs},
    \begin{align}\label{eq:nj_nl_lb2}
         & \ 2\|W\|^2\sum_{i=1}^{n}\frac{\partial\pi}{\partial p_i}\npa(W)^\T\nabla^2p_i(W)\nra(W) \nonumber \\
        = & \ 2(L-1)\|W\|\enVert{\na(W)}\cos(\theta)\ip{\npa(W)}{\na(W)} \nonumber \\
        = & \ 2(L-1)\|W\|\enVert{\na(W)}^3\cos(\theta)\sin(\theta)^2 \nonumber \\
        = & 2L(L-1)\enVert{\na(W)}^2\sin(\theta)^2\beta(W).
    \end{align}
    In addition, the proof of \Cref{fact:margins} shows that
    $\enVert{\pi(p)}_1\le2$, and \Cref{fact:grad_hom} ensures that
    $\|\nabla^2f\|_\sigma$ has a uniform bound $K_\sigma$ on the unit sphere,
    therefore
    \begin{align}\label{eq:nj_nl_lb3}
        \|W\|^2\sum_{i=1}^{n}\frac{\partial\pi}{\partial p_i}\npa(W)^\T\nabla^2p_i(W)\npa(W) & \le2\|W\|^2\enVert{\na(W)}^2\sin(\theta)^2\cdot K_\sigma\|W\|^{L-2} \nonumber \\
         & =2K_\sigma\|W\|^L\enVert{\na(W)}^2\sin(\theta)^2.
    \end{align}
    Combining \cref{eq:nj_nl,eq:nj_nl_lb,eq:nj_nl_lb1,eq:nj_nl_lb2,eq:nj_nl_lb3}
    gives
    \begin{align*}
        \frac{\ip{\nabla\cJ(W)}{-\nL(W)}}{-\ell'\del{\alpha(W)}}\le \frac{4\del{K_\sigma\|W\|^L+L(L-1)\beta(W)}\enVert{\na(W)}^2}{\|W\|^{2L}}\sin(\theta)^2.
    \end{align*}
    Invoking \Cref{fact:margins_sizes} then gives
    \begin{align*}
        \ip{\nabla\cJ(W)}{-\nL(W)}\le -K\ell'\del{\alpha(W)}\|W\|^{L-2}\sin(\theta)^2
    \end{align*}
    for some constant $K>0$.
\end{proof}

The following result helps us control $\theta_t$.
\begin{lemma}\label{fact:theta_int}
    Under the same condition as \Cref{fact:margins_sizes} and \Cref{cond:init},
    it holds that
    \begin{align*}
        \int_0^\infty-\ell'\del{\alpha(W_t)}\|W_t\|^{L-2}\tan(\theta_t)^2\dif t<\infty.
    \end{align*}
\end{lemma}
\begin{proof}
    Recall that $\ta_t=\alpha(W_t)/\|W_t\|^L$ is nondecreasing with a limit $a$,
    and thus $\dif\ta_t/\dif t$ is integrable.
    Now \Cref{fact:ta_zeta,fact:ta_cs,fact:alpha_cs} imply that
    \begin{align*}
        \frac{\dif\ta_t}{\dif t}\ge\enVert{\nabla_\perp\ta(W_t)}{\enVert{\nabla_\perp\cL(W_t)}} & =\frac{\enVert{\npa(W_t)}{\enVert{\nabla_\perp\cL(W_t)}}}{\|W_t\|^L}=\frac{-\ell'\del{\alpha(W_t)}\enVert{\npa(W_t)}^2}{\|W_t\|^L},
    \end{align*}
    and moreover
    \begin{align*}
        \enVert{\npa(W_t)}=\enVert{\nra(W_t)}\tan(\theta_t)=\frac{L\beta(W_t)}{\|W_t\|}\tan(\theta_t).
    \end{align*}
    Therefore
    \begin{align*}
        \frac{\dif\ta_t}{\dif t}\ge-\ell'\del{\alpha(W_t)}\cdot L^2\tan(\theta_t)^2 \frac{\beta(W_t)^2}{\|W_t\|^{L+2}}.
    \end{align*}
    Since $\beta(W_t)/\|W_t\|^L$ is bounded due to \Cref{fact:margins_sizes},
    the proof is finished.
\end{proof}

Now we can prove \Cref{fact:alignment}.
\begin{proof}[Proof of \Cref{fact:alignment}]
    Fix an arbitrary $\epsilon\in(0,1)$, and let $J_t$ denote $J(W_t)$.
    Recall that $\lim_{t\to\infty}\alpha(W_t)/\|W_t\|^L=a$.
    \Cref{fact:margins} then implies $\lim_{t\to\infty}\beta(W_t)/\|W_t\|^L=a$,
    and thus we can find $t_1$ such that for any $t>t_1$,
    \begin{align}\label{eq:beta_rho_range}
        a\del{1-\frac{\epsilon}{6}}<\frac{\beta(W_t)}{\|W_t\|^L}=\frac{1}{L}\ip{\frac{\na(W_t)}{\|W_t\|^{L-1}}}{\frac{W_t}{\|W_t\|_F}}<a\del{1+\frac{\epsilon}{6}}.
    \end{align}
    Moreover, \Cref{fact:chain,fact:der_norm,fact:theta_int} imply that there
    exists $t_2$ such that for any $t'>t>t_2$,
    \begin{align}\label{eq:jt_dif}
        J_{t'}-J_{t}<\del{\frac{aL\epsilon}{6}}^2.
    \end{align}
    \citep[Corollary C.10]{kaifeng_jian_margin} implies that there exists
    $t_3>\max\{t_1,t_2\}$ such that
    \begin{align}\label{eq:cos_t3}
        \frac{1}{\cos(\theta_{t_2})^2}-1<\frac{\epsilon}{3},\quad\textrm{and thus}\quad \frac{1}{\cos(\theta_{t_2})}<1+\frac{\epsilon}{6}.
    \end{align}
    We claim that $\delta_t<1+\epsilon$ for any $t>t_3$.

    To see this, note that \cref{eq:beta_rho_range,eq:cos_t3} imply
    \begin{align*}
        \sqrt{J_{t_2}}=\frac{\enVert{\na(W_{t_2})}}{\|W_{t_2}\|^{L-1}}<aL\del{1+\frac{\epsilon}{6}}\frac{1}{\cos(\theta_{t_2})}<aL\del{1+\frac{\epsilon}{6}}^2<aL\del{1+\frac{\epsilon}{2}}.
    \end{align*}
    Moreover, using \cref{eq:jt_dif}, for any $t>t_2$,
    \begin{align*}
        \sqrt{J_t}=\sqrt{J_{t_2}+J_t-J_{t_2}}<\sqrt{J_{t_2}+\del{\frac{\gamma L\epsilon}{6}}^2}<\sqrt{J_{t_2}}+\frac{aL\epsilon}{6}<aL\del{1+\frac{2\epsilon}{3}},
    \end{align*}
    and thus
    \begin{align*}
        \frac{1}{\cos(\theta_t)}=\frac{\sqrt{J_t}}{L\beta(W_t)/\|W_t\|^L}<\frac{aL\del{1+\nicefrac{2\epsilon}{3}}}{aL(1-\nicefrac{\epsilon}{6})}<1+\epsilon.
    \end{align*}
    Since $\epsilon$ is arbitrary, we have $\lim_{t\to\infty}\theta_t=0$.

    If all $p_i$ are $C^2$, then the above proof holds without definability: it
    is only used in \cref{eq:jt_dif} to ensure the chain rule, which always
    holds for $C^2$ functions.
\end{proof}

\section{Global margin maximization proofs for \Cref{sec:margins}}
\label{app_sec:margins}

This section often works with subscripted subsets of parameters, for instance per-layer matrices
$(A_1(t),\ldots,A_L(t))$, or per-node weights $(w_1(t),\ldots,w_m(t))$;
to declutter slightly, we will drop ``$(t)$'' throughout when it is otherwise clear.

First, a technical lemma regarding directional convergence and alignment
properties inherited by these subsets of $W_t$.  This will be used in both the deep linear
case and in the 2-homogeneous case.

\begin{lemma}
  \label{fact:alignment:layers}
  Suppose the conditions for \Cref{fact:dir,fact:alignment} hold.
  Let $(U_1(t),\ldots,U_r(t))$ be any partition of $W_t$,
  and set $s_j(t) := \|U_j(t)\|^L / \|W_t\|^L$.
  Then $s(t)$ converges to some $\bars$,
  and for each $j$,
  \[
    \lim_{t\to\infty} \frac{ \|U_j\|\cdot \|\nabla_{U_j} \cL(W)\|}{\|W\|\cdot\|\nabla_W \cL(W)\|}
    =
    \lim_{t\to\infty} \frac{ \ip{U_j}{-\nabla_{U_j} \cL(W)}}{\|W\|\cdot\|\nabla_W \cL(W)\|},
  \]
  and moreover $\bars_j > 0$ implies
  \[
    \lim_{t\to\infty} \frac{\|U_j\|}{\|W\|}
    =
    \lim_{t\to\infty} \frac{\|\nabla_{U_j} \cL(W)\|}{\|\nabla_W \cL(W)\|}
    =
    \lim_{t\to\infty} \frac{\|\nabla_{U_j} \alpha(W)\|}{\|\nabla_W \alpha(W)\|}
    =
    \bars_j^{1/L},
  \]
  and
  \[
    \lim_{t\to\infty}
    \frac{\ip{U_j}{-\nabla_{U_j}\cL(W)}}{\|U_j\|\cdot\|\nabla_{U_j}\cL(W)\|}
    =
    \lim_{t\to\infty}
    \frac{\ip{U_j}{\nabla_{U_j}\alpha(W)}}{\|U_j\|\cdot\|\nabla_{U_j}\alpha(W)\|}
    =
    1,
  \]
  and
  \[
    \lim_{t\to\infty} \frac {\ip{U_j}{\nabla_{U_j}\alpha(W)}} {\|U_j\|^L}
    =
    \lim_{t\to\infty} \frac {\|\nabla_{U_j}\alpha(W)\|} {\|U_j\|^{L-1}}
    =
    a \bar s^{(2-L)/L} L.
  \]
\end{lemma}
\begin{proof}
  First note that $s(t)$ converges since $W_t/\|W_t\|$ converges,
  and alignment grants
  \begin{align}\label{eq:bars_j_limit}
      \bars_j^{1/L}
      =
      \lim_{t\to\infty} \frac{\|U_j\|}{\|W\|}
      =
      \lim_{t\to\infty} \frac{\|\nabla_{U_j} \cL(W)\|}{\|\nabla_W \cL(W)\|}.
  \end{align}
  By directional convergence (cf. \Cref{fact:dir}),
  alignment (cf. \Cref{fact:alignment}),
  and Cauchy-Schwarz,
  \begin{align*}
    -1
    &=
    \lim_{t\to\infty} \frac {\ip{W}{\nabla_W\cL(W)}}{\|W\|\cdot\|\nabla_W\cL(W)\|}
    \\
    &=
    \lim_{t\to\infty} \frac {\sum_j \ip{U_j}{\nabla_{U_j}\cL(W)}}{\|W\|\cdot\|\nabla_W\cL(W)\|}
    \\
    &\geq
    - \lim_{t\to\infty} \frac {\sum_j \|U_j\|\cdot \|\nabla_{U_j}\cL(W)\|}{\|W\|\cdot\|\nabla_W\cL(W)\|}
    \\
    &\geq
    - \lim_{t\to\infty} \frac {\sqrt{\sum_j \|U_j\|^2} \cdot \sqrt{\sum_j \|\nabla_{U_j}\cL(W)\|^2}}
    {\|W\|\cdot\|\nabla_W\cL(W)\|}
    = -1,
  \end{align*}
  which starts and ends with $-1$ and is thus a chain of equalities.
  Applying \cref{eq:bars_j_limit} and he equality case of Cauchy-Schwarz to each
  $j$ with $\bars_j>0$,
  \begin{align*}
    \bars_j^{2/L}
    &=
    \lim_{t\to\infty} \frac {\|U_j\|\cdot \|\nabla_{U_j}\cL(W)\|}{\|W\|\cdot\|\nabla_W\cL(W)\|}
    =
    \lim_{t\to\infty} \frac {\ip{U_j}{-\nabla_{U_j}\cL(W)}}{\|W\|\cdot\|\nabla_W\cL(W)\|}
    \\
    &=
    \lim_{t\to\infty}
    \frac {\ip{U_j}{-\nabla_{U_j}\cL(W)}}{\|U_j\|\cdot\|\nabla_{U_j}\cL(W)\|}
    \del{
      \frac{\|U_j\|\cdot\|\nabla_{U_j}\cL(W)\|} {\|W\|\cdot\|\nabla_W\cL(W)\|}
    }
    \\
    &= \bars_j^{2/L}
    \lim_{t\to\infty}
    \frac {\ip{U_j}{-\nabla_{U_j}\cL(W)}}{\|U_j\|\cdot\|\nabla_{U_j}\cL(W)\|},
  \end{align*}
  and thus
  \[
    \lim_{t\to\infty}
    \frac {\ip{U_j}{-\nabla_{U_j}\cL(W)}}{\|U_j\|\cdot\|\nabla_{U_j}\cL(W)\|}
    =
    1.
  \]
  The preceding statements used $\cL(W)$; to obtain the analogous statements
  with $\alpha(W)$, note since $\ell' <0$ that
  \[
    \frac {\nabla_{U_j} \alpha(W)}{\|\nabla_{U_j} \alpha(W)\|}
    =
    \frac
    {\nabla_{U_j} \cL(W) / \ell'(\alpha(W))}
    {\|\nabla_{U_j} \cL(W) / \ell'(\alpha(W))\|}
    =
    \frac
    {-\nabla_{U_j} \cL(W)}
    {\|\nabla_{U_j} \cL(W)\|}.
  \]

  For the final claim,
  note \Cref{fact:alignment} and \cref{eq:beta_limit} imply that
  \begin{align*}
    \lim_{t\to\infty}\frac{\|\na(W_t)\|}{\|W_t\|^{L-1}}
    =\lim_{t\to\infty}\frac{\ip{\na(W_t)}{W_t}}{\|W_t\|^L}
    =aL > 0,
  \end{align*}
  and when $\bars_j > 0$,
  \begin{align*}
    \lim_{t\to\infty} \frac {\ip{U_j}{\nabla_{U_j}\alpha(W_t)}} {\|U_j\|^L}
    &=
    \lim_{t\to\infty} \frac {\|U_j\| \cdot \|\nabla_{U_j}\alpha(W)\|} {\|U_j\|^L}
    =
    \lim_{t\to\infty} \frac {\|\nabla_{U_j}\alpha(W_t)\|} {\|U_j\|^{L-1}}
    \\
    &=
    \lim_{t\to\infty} \frac {\bar s^{1/L} \|\nabla_W \alpha(W)\|}{\bar s^{(L-1)/L}\|W\|^{L-1}}
    =
    a L \bar s^{(2-L)/L}.
  \end{align*}
\end{proof}

Applying the preceding lemma to network layers, we handle the deep
linear case as follows.

\begin{proof}[Proof of \Cref{fact:deep_linear}]
  For convenience, write $A_j$ instead of $A_j(t)$ when time $t$ is clear,
  and also $u:=A_j\cdots{}A_1$ and
  $\nabla_u \cL(W) = \sum_i \ell'(y_i u^\T x_i) y_i x_i$.
  By this notation,
  \begin{align*}
    \nabla_{A_j} \cL(W)
    &= \sum_i \ell'(y_i u^\T x_i) y_i (A_L\cdots{}A_{j+1})^\T (A_{j-1}\cdots A_1 x_i)^\T
    \\
    &= (A_L\cdots{}A_{j+1})^\T (A_{j-1}\cdots A_1 \nabla_u \cL(W))^\T,
  \end{align*}
  where $(A_L\cdots{}A_{j+1})^\T$ is a column vector, and
  $(A_{j-1}\cdots A_1 \nabla_u \cL(W))^\T$ is a row vector,
  and moreover
  $\ip{A_j}{\nabla_{A_j} \cL(W)} = \ip{u}{\nabla_u\cL(W)}$,
  where this last inner product does not depend on $j$.

  Applying the subset-alignment of \Cref{fact:alignment:layers}
  to layers $(A_j,\ldots,A_1)$
  gives, for each $j$,
  \[
    \bar s_j^{2/L}
    =
    \lim_{t\to\infty} \frac{ \|A_j\|\cdot \|\nabla_{A_j} \cL(W)\|}{\|W\|\cdot\|\nabla_W \cL(W)\|}
    =
    \lim_{t\to\infty} \frac{ \ip{A_j}{-\nabla_{A_j} \cL(W)}}{\|W\|\cdot\|\nabla_W \cL(W)\|}
    =
    \lim_{t\to\infty} \frac{ -\ip{u}{\nabla_u \cL(W)} }{\|W\|\cdot\|\nabla_W \cL(W)\|},
  \]
  whereby $\bars_j$ is independent of $j$,
  which can only mean $\bar s_j^{2/L} = 1/L > 0$ for all $j$,
  but more importantly
  $\|A_j(t)\|\to\infty$ for all $j$.
  By \Cref{fact:alignment:layers}, this means all layers align with their gradients.

  Next it is proved by induction from $A_L$ to $A_1$ that there exist unit vectors
  $v_0,\ldots,v_L$ with $v_L =1$  and $A_j/\|A_j\| = v_j v_{j-1}^\T$.
  The base case $A_L$ holds
  immediately, since $A_L$ is a row vector, meaning we can choose $v_L:= 1$ and
  $v_{L-1} := A_L^\T/\|A_L\|$ since $A_L$ converges in direction.
  For the inductive step $A_j$ with $j<L$, note
  \begin{align*}
     \lim_{t\to\infty} \frac {\nabla_{A_j}\cL(W)}{\|\nabla_{A_j}\cL(W)\|}
     &=
     \lim_{t\to\infty} \frac
     {(A_L\cdots A_{j+1})^\T (A_{j-1}\cdots A_1\nabla_u \cL(W))^\T}
     {\enVert{(A_L\cdots A_{j+1})^\T (A_{j-1}\cdots A_1\nabla_u \cL(W))^\T}}
     \\
     &=
     \lim_{t\to\infty} \frac
     {(A_L\cdots A_{j+1})^\T (A_{j-1}\cdots A_1\nabla_u \cL(W))^\T}
     {\enVert{(A_L\cdots A_{j+1})}\enVert{(A_{j-1}\cdots A_1\nabla_u \cL(W))}}
     \\
     &=
     \lim_{t\to\infty}
     \frac
     {(v_L v_{L-1}^\T \cdots v_{j+1} v_j^\T)^\T (A_{j-1}\cdots A_1\nabla_u \cL(W))^\T}
     {\enVert{A_{j-1}\cdots A_1\nabla_u \cL(W)}}
     \\
     &=
     \lim_{t\to\infty}
     \frac
     {v_j (A_{j-1}\cdots A_1\nabla_u \cL(W))^\T}
     {\enVert{A_{j-1}\cdots A_1\nabla_u \cL(W)}}.
  \end{align*}
  Since $v_j$ is a fixed unit vector and since $\nabla_{A_j} \cL(W)$ converges
  in direction, the row vector part of the above expression must also converge to
  some fixed unit vector $v_{j-1}^\T$, namely
  \[
    \lim_{t\to\infty} \frac {\nabla_{A_j}\cL(W)}{\|\nabla_{A_j}\cL(W)\|} = -v_j v_{j-1}^\T
    \qquad
    \text{where }
    v_{j-1} :=
    -\lim_{t\to\infty}
    \frac
    {A_{j-1}\cdots A_1\nabla_u \cL(W)}
    {\enVert{A_{j-1}\cdots A_1\nabla_u \cL(W)}}.
  \]
  Since $A_j$ and $-\nabla_{A_j} \cL(W)$ asymptotically align as above,
  then $\nicefrac{A_j}{\|A_j\|} \to v_j v_{j-1}^\T$.

  Now consider $v_0$ and $u$, where it still needs to be shown that $v_0 = u/\|u\|$.
  To this end, note
  \begin{align*}
    1
    &\geq
    \lim_{t\to\infty} \frac {v_0^\T u}{\|u\|}
    =
    \lim_{t\to\infty} \frac {v_0^\T A_L \cdots A_1}{\|A_L \cdots A_1\|}
    \geq
    \lim_{t\to\infty}
    \del{\frac {\|A_L\|\cdots\|A_1\|}{\|A_L\|_\sigma\cdots \|A_{2}\|_\sigma\|A_1\|}}
    v_0^\T \del[2]{v_L v_{L-1}^\T \cdots v_1 v_0^\T}
    \\
    &=
    v_0^\T v_0 = 1,
  \end{align*}
  whereby $u/\|u\| = v_0$.
By a similar calculation,
  \[
    -1
    = \lim_{t\to\infty} \frac {\ip{A_1}{\nabla_{A_1}\cL(W)}}{\|A_1\|\cdot\|\nabla_{A_1}\cL(W)\|}
    = \lim_{t\to\infty} \frac {\ip{u}{\nabla_{u}\cL(W)}}{\|A_1\|\cdot\|A_L\cdots A_2 \nabla_{u}\cL(W)\|}
    = \lim_{t\to\infty} \frac {\ip{u}{\nabla_{u}\cL(W)}}{\|u\|\cdot\|\nabla_{u}\cL(W)\|},
  \]
  which means $u/\|u\|$ asymptotically satisfies the optimality conditions for the
  optimization problem
  \[
    \min_{\|w\| \leq 1} \frac 1 {\|A_L\cdots A_1\|} \sum_i \ell\del{ \|A_L\cdots{}A_1\| y_i x_i^\T w },
  \]
  which is asymptotically solved by the unique maximum margin vector $\bar u$,
  which is guaranteed to exist since the data is linearly separable
  thanks to $\cL(W_0)<\ell(0)$.
\end{proof}

Before moving on to the 2-homogeneous case, we first
produce another technical lemma, which we will use to control
\emph{dual variables} $q_i(t) := \partial\alpha/\partial p_i(W_t)$,
which also appear in \Cref{fact:covering}.

\begin{lemma}
  \label{fact:dual_opt}
  Every accumulation point $\barq$ of $\setdef{ q(t) }{t \in \N}$ satisfies
  $\barq \in \Delta_n$ and
  \[
    \sum_i \barq_i \ip{\frac {W}{\|W\|}}{\frac {\nabla_W p_i(W)}{\|W\|^{L-1}}}
    =
    \lim_{t\to\infty}\ip{\frac {W}{\|W\|}}{\frac {\nabla_W \alpha(W)}{L\|W\|^{L-1}}}
    =
    \min_i \lim_{t\to\infty} \frac {p_i(W_t)}{\|W_t\|^L}
    = a.
  \]
\end{lemma}
\begin{proof}
  By \Cref{fact:alpha_cs,fact:margins},
  \[
    \lim_{t\to\infty} \frac {\alpha(W_t)}{\|W_t\|^L}
    =
    \lim_{t\to\infty} \ip{ \frac {W_t}{\|W_t\|} }{ \frac {\nabla_W \alpha(W_t)}{L\|W_t\|^{L-1}} }
    =
    \lim_{t\to\infty}\min_i \frac {p_i(W_t)}{\|W_t\|^L} = a
    =
    \min_i \lim_{t\to\infty} \frac {p_i(W_t)}{\|W_t\|^L} = a.
  \]
  Moreover, since $\lim_{z\to\infty} \frac {\llog(z)}{\lexp(z)} = 1$ and since $a>0$ and
  $\|W_t\|\to \infty$, then $q(t)$ is asymptotically within the simplex,
  meaning $\lim_{t\to\infty} \min_{q'\in\Delta_n} \|q(t) - q'\| = 0$.  Consequently,
  every accumulation point $\barq$ of $\{ q(t) : t\in \N \}$ satisfies $\barq \in \Delta_n$,
  and
  \[
    \sum_i \barq_i\lim_{t\to\infty} \ip{\frac {W}{\|W\|}}{\frac {\nabla_W p_i(W)}{\|W\|^{L-1}}}
    =
    \lim_{t\to\infty} \ip{\frac {W}{\|W\|}}{\frac {\nabla_W \alpha_i(W)}{L\|W\|^{L-1}}}
    =
    \lim_{t\to\infty} \min_i \frac {p_i(W_t)}{\|W_t\|^L}
    = a.
  \]
\end{proof}

With this in hand, we can handle the 2-homogeneous case.

\begin{proof}[Proof of \Cref{fact:covering}]
  Applying \Cref{fact:alignment:layers} to the per-node weights $(w_1,\ldots,w_m)$,
  a limit $\bar s$ exists and due to 2-homogeneity satisfies $\bar s \in \Delta_m$.
  Whenever, $\bar s_j > 0$, then
  \begin{align*}
    \lim_{t\to\infty}
    2 \sum_i q_i(t) \varphi_{ij}(\theta_j(t))
    &=
    \lim_{t\to\infty}
    \ip{ \theta_j(t) }{ \sum_i q_i(t) \nabla_\theta \varphi_{ij}(\theta_j(t)) }
    \\
    &=
    \lim_{t\to\infty}
    \ip{ \frac {w_j(t)}{\|w_j(t)\|} }{ \frac {\nabla_{w_j} \alpha(W_t)}{\|w_j(t)\|} }
    =
    2 a \bars^{0/2} = 2 a.
  \end{align*}
  Consequently, this means that either $\bars_j >0 $
  and $\lim_{t\to\infty} \sum_i q_i(t) \varphi_{ij}(\theta_j(t)) = a$,
  or else $\bars_j = 0$ and by the choice $\btheta_j = 0$ then
  $\lim_{t\to\infty} \sum_i q_i(t) \varphi_{ij}(\theta_j(t)) = 0$.
  In particular, this means $\bar s_j>0$ iff $\btheta_j$ attains the maximal value $a$,
  meaning $\bars$ satisfies the Sion primal optimality conditions
  for the saddle point problem over the fixed points $(\btheta_1,\ldots,\btheta_m)$
  \citep[Proposition D.3]{chizat_bach_imp}.

  Now consider the dual variables $q_i(t) = \partial\alpha/\partial p_i(W_t)$.
  By \Cref{fact:dual_opt}, any accumulation point $\barq$ is an element of
  $\Delta_n$ and moreover is supported on those examples $i$ minimizing
  $p_i(\oW)$,
  which means $\bar q$ satisfies the Sion dual optimality conditions
  for the margin saddle point problem again over fixed points $(\btheta_1,\ldots,\btheta_m)$
  \citep[Proposition D.3]{chizat_bach_imp}.
  Thus applying the Sion Theorem
  over discrete domain
  $(\btheta_1,\ldots,\btheta_m)$
  to the primal-dual optimal pair
  $(\bars, \barq)$ gives
  \[
    \sum_i\barq_i \sum_j \bars_j \varphi_{ij}(\btheta_j)
    = \min_{q\in\Delta_n} \max_{s\in\Delta_m} \sum_i q_i \sum_j s_j \varphi_{ij}(\btheta_j)
    = \min_i \max_{s\in\Delta_m} \sum_j s_j \varphi_{ij}(\btheta_j),
  \]
  and directional convergence of $\tW_t$ combined with definition of $\barq$ gives
  \begin{align*}
    \lim_{t\to\infty}
    \sum_i q_i(t) s_j(t) \varphi_{ij}(\theta_j(t))
    = \sum_i \barq_i \sum_j \bars_j\varphi_{ij}(\btheta_j(t)).
  \end{align*}
  Since $\barq$ was an arbitrary accumulation point, it holds in general that
  \[
    \lim_{t\to\infty}
    \sum_i q_i(t) s_j(t) \varphi_{ij}(\theta_j(t))
    = \min_{q\in\Delta_n} \max_{s\in\Delta_m} \sum_i q_i \sum_j s_j \varphi_{ij}(\btheta_j).
  \]

  Now for the global guarantee.
  Fix $t_0$ for now, and consider
  $(\theta_j)_{j=1}^m = (\theta_j(t_0))_{j=1}^m$ and their
  cover guarantee.
  For any signed measure $\nu$ on $\S^{d-1}$,
  we can partition $\S^{d-1}$ twice so that $(\nu(\theta_1), \nu(\theta_3),\ldots)$
  partitions the negative mass of $\nu$ by associating it with the closest element
  amongst $(\theta_1,\theta_3,\ldots)$, all of which have negative coefficient in $\varphi_{ij}$,
  and also the positive mass of $\nu$ into $(\nu(\theta_2),\nu(\theta_4),\dots)$;
  in this way, we now have converted $\nu$ on $\S^{d-1}$ into a discrete measure on
  $(\theta_1,\ldots,\theta_m)$.
  Noting that $z\mapsto \max\{0,z\}^2$ is $2$-Lipschitz over $[-1,1]$,
  and therefore for any $i$ and any unit norm $\theta,\theta'$ that
  \[
    |\varphi_{ij}(\theta) - \varphi_{ij}(\theta')|
    =
    \envert{ \max\{0,x_i^\T\theta\}^2 - \max\{0,x_i^\T\theta'\}^2}
    \leq 2 \envert{ x_i^\T\theta - x_i^\T \theta' }
    \leq 2 \|\theta - \theta'\|,
  \]
  then, letting ``$\theta\to\theta_j$'' denote the subset of $\S^{d-1}$ associated
  with $\theta_j$ as above (positively or negatively),
  and letting $\varphi_i(\theta) := y_i \max\{0,x_i^\T\theta\}^2$,
  for any $q$,
  \begin{align*}
    &
    \envert{
      \sum_i q_i \int \varphi_i(\theta)\dif \nu(\theta)
      -
      \sum_i q_i \sum_j \nu(\theta_j) \varphi_{ij}(\theta_j)
    }
    \\
    &=
    \envert{
      \sum_i q_i \sum_j \int_{\theta\to\theta_j} \varphi_i(\theta)\dif \nu(\theta)
      -
      \sum_i q_i \sum_j \nu(\theta_j) \varphi_{ij}(\theta_j)
    }
    \\
    &\leq
    \sum_i q_i \int_{\theta\to\theta_j} \sum_j
    \envert{
      \varphi_{ij}(\theta)
      -
      \varphi_{ij}(\theta_j)
    }
    \dif|\nu|(\theta)
    \\
    &\leq
    2 \sum_i q_i \int_{\theta\to\theta_j} \sum_j
    \enVert{
      \theta
      -
      \theta_j
    }
    \dif|\nu|(\theta)
    \leq
    2\eps.
  \end{align*}
  Thus
  \begin{align*}
    \min_{q\in\Delta_n} \max_{p\in\Delta_m}
    \sum_i q_i \sum_j p_j \varphi_{ij}(\theta_j)
    &
    \leq
    \min_{q\in\Delta_n} \max_{\nu\in\cP(\S^{d-1})}
    \sum_i q_i \int \varphi_i(\theta)\dif \nu(\theta)
    \\
    &
    \leq
    2 \eps
    +
    \min_{q\in\Delta_n} \max_{p\in\Delta_m}
    \sum_i q_i \sum_j p_j \varphi_{ij}(\theta_j).
  \end{align*}
  Next, for any $q\in \Delta_n$ and $s\in\Delta_m$, using the first part of the cover condition,
  \begin{align*}
    \sum_{i,j} q_i s_j (\varphi_{ij}(\btheta_j) - \varphi_{ij}(\theta_j(t_0))
    &\leq
    \sum_{i,j} q_i s_j | \varphi_{ij}(\btheta_j) - \varphi_{ij}(\theta_j(t_0) |
\leq
    2 \sum_{i,j} q_i s_j \| \btheta_j - \theta_j(t_0) \|
\leq
    2 \eps,
  \end{align*}
  thus
  \begin{align*}
    \lim_{t\to\infty} \sum_{i,j} q_i s_j \varphi_{ij}(\theta_j)
    &=
    \min_{q\in\Delta_n} \max_{s\in\Delta_m} \sum_{i,j} q_i s_j \varphi_{ij}(\btheta_j)
    \\
    &=
    \min_{q\in\Delta_n} \max_{s\in\Delta_m}
    \sbr{ \sum_{i,j} q_i s_j \varphi_{ij}(\theta_j(t_0))
      - \sum_{i,j} q_i s_j \del{ \varphi_{ij}(\theta_j(t_0)) - \varphi_{ij}(\btheta_j) }
    }
    \\
    &\geq
    \min_{q\in\Delta_n} \max_{s\in\Delta_m}
    \sum_{i,j} q_i s_j \varphi_i(\theta_j(t_0))
    - 2\eps
    \\
    &\geq
    \min_{q\in\Delta_n} \max_{\nu\in\cP(\S^{d-1})} \sum_i q_i \int \varphi_i(\theta)\dif\nu(\theta)
    - 4\eps.
  \end{align*}
\end{proof}

\end{document}